\documentclass[twoside]{article}
\usepackage[accepted]{aistats2024}

\usepackage{amsmath,amsfonts,bm}

\def\eqref#1{equation~\ref{#1}}

\def\1{\bm{1}}

\def\eps{{\epsilon}}

\newcommand{\ps}[1]{\langle #1 \rangle}

\DeclareMathAlphabet{\mathsfit}{\encodingdefault}{\sfdefault}{m}{sl}
\SetMathAlphabet{\mathsfit}{bold}{\encodingdefault}{\sfdefault}{bx}{n}

\newcommand{\E}{\mathbb{E}}

\newcommand{\R}{\mathbb{R}}

\DeclareMathOperator*{\argmax}{arg\,max}
\DeclareMathOperator*{\argmin}{arg\,min}

\usepackage[round]{natbib}
\renewcommand{\bibname}{References}
\renewcommand{\bibsection}{\subsubsection*{\bibname}}

\usepackage[pdfencoding=auto, psdextra, pagebackref=true]{hyperref}
\usepackage{xcolor}
\hypersetup{
    colorlinks,
    linkcolor={blue!50!black},
    citecolor={blue!50!black},
    urlcolor={blue!80!black}
}

\usepackage{url}
\usepackage{algorithm}
\usepackage[noend]{algpseudocode}
\usepackage{booktabs}
\usepackage{xcolor}
\usepackage{multirow}
\usepackage{amssymb}
\usepackage{amsthm}
\usepackage{amsmath}
\usepackage{mathtools}
\usepackage{makecell}
\usepackage{amssymb}
\usepackage{float}
\usepackage{pifont}
\usepackage[sort,capitalise]{cleveref}
\usepackage{enumitem}
\usepackage{balance}

\newtheorem{theorem}{Theorem}

\newtheorem{lemma}{Lemma}
\newtheorem{proposition}{Proposition}
\newtheorem{definition}{Definition}
\newtheorem{assumption}{Assumption}

\newtheorem{remark}{Remark}

\usepackage{wrapfig}
\usepackage[skip=6pt]{caption}

\addtocontents{toc}{\protect\setcounter{tocdepth}{0}}

\crefname{assumption}{Assumption}{Assumptions}
\crefname{theorem}{Theorem}{Theorems}
\crefname{lemma}{Lemma}{Lemma}
\crefname{corollary}{Corollary}{Corollaries}
\crefname{proposition}{Proposition}{Propositions}
\crefname{algorithm}{Algorithm}{Algorithms}
\crefname{section}{section}{sections}
\crefname{equation}{}{}

\newtheorem*{T1}{Theorem~\ref{thm:main-exact-gradients}}
\newtheorem*{T2}{Theorem~\ref{thm:main-finite-sample}}

\begin{document}

\twocolumn[

\aistatstitle{Independent Learning in Constrained Markov Potential Games}

\aistatsauthor{ Philip Jordan \And Anas Barakat \And  Niao He }

\aistatsaddress{ ETH Z\"urich \And  ETH Z\"urich \And ETH Z\"urich } ]

\begin{abstract}
Constrained Markov games offer a formal mathematical framework for modeling multi-agent reinforcement learning problems where the behavior of the agents is subject to constraints. 
In this work, we focus on the recently introduced class of constrained Markov Potential Games. 
While centralized algorithms have been proposed for solving such constrained games, the design of converging independent learning algorithms tailored for the constrained setting remains an open question.
We propose an independent policy gradient algorithm for learning approximate constrained Nash equilibria: Each agent observes their own actions and rewards, along with a shared state.
Inspired by the optimization literature, our algorithm performs proximal-point-like updates  augmented with a regularized constraint set. Each proximal step is solved inexactly using a stochastic switching gradient algorithm. 
Notably, our algorithm can be implemented independently without a centralized coordination mechanism requiring turn-based agent updates.
Under some technical constraint qualification conditions, we establish convergence guarantees towards constrained approximate Nash equilibria. 
We perform simulations to illustrate our results. 
\end{abstract}

\section{INTRODUCTION}\label{sec:introduction}

In multi-agent reinforcement learning (RL), several agents interact within a shared dynamic and uncertain environment evolving over time depending on the individual strategic decisions of all the agents. Each agent aims to maximize their own individual reward which may however depend on all players'\footnote{We will use \textit{player} and \textit{agent} interchangeably.} decisions. Besides reward maximization, agents may also contend with satisfying constraints that are often dictated by multi-agent RL applications. Prominent such real-world applications include multi-robot control on cooperative tasks~\citep{gu-et-al23} as well as autonomous driving~\citep{shalev-shwartz-et-al16,liu-et-al23autonomous-driving} where physical system constraints and safety considerations such as collision avoidance are of primary importance. 
In other applications, agents may be subject to soft constraints such as average users' total latency thresholds in wireless networks or average power constraints in signal transmission. 
Each agent seeks to maximize their reward while also accounting for constraints which are coupled among agents. 
Constrained Markov games~\citep{altman_constrained_2000} offer a mathematical framework to model multi-agent RL problems incorporating coupled constraints. 

In this work, we focus on a particular class of structured constrained Markov games: constrained Markov Potential Games (CMPGs). Recently introduced in~\citet{alatur_provably_2023} to incorporate constraints, CMPGs naturally extend the class of Markov Potential Games~(MPGs) that has been actively investigated in the last few years~\citep{macua-et-al18,leonardos_global_2021,fox-et-al22,zhang-et-al21grad-play,song-mei-bai22,ding_independent_2022,zhang-et-al22softmax,maheshwari_independent_2023,zhou-et-al23networked-mpg}. 
Interestingly, this class of games is a class of mixed cooperative/competitive Markov games including pure identical interest Markov games (in which all the reward and cost functions of the agents are identical) as a particular case. 
The ability to cooperate between learning agents is crucial to improve their joint welfare and achieve social welfare for artificial intelligence (see~\citet{dafoe-et-al20,dafoe-et-al21} for an extensive discussion about the need for promoting cooperative AI). 

Independent learning has recently attracted increasing attention thanks to its versatility as a learning protocol. We refer the reader to a recent nice survey on the topic~\citep{ozdaglar-sayin-zhang21survey}. 
In this protocol, agents can only observe the realized state and their own reward and action in each stage to individually optimize their return. 
In particular, each agent does not observe actions or policies from any other agent. This protocol offers several advantages including the following aspects: (a)~Scaling: independent learning dynamics do not scale exponentially with the number of players in the game (also known as the curse of multi-agents); (b)~Privacy protection: agents may avoid sharing their local data and information to protect their privacy and autonomy; (c)~Communication cost: a central node that can bidirectionally communicate with all agents may not exist or may be too expensive to afford. Therefore, this protocol is particularly appealing in several applications where agents need to make decisions independently, in a decentralized manner. For example, dynamic load balancing, which consists in evenly assigning clients to servers in distributed computing, demands for learning algorithms that minimize communication overhead to enable low-latency response times and scalability across large data centers. This task has been modeled as an MPG~\citep{yao2022learning}. In other applications such as the pollution tax model and distributed energy marketplace detailed in section~\ref{sec:experiments}, coordination is inherently ruled out due to the competitive nature of the players' interactions. Independent learning algorithms have been proposed for unconstrained multi-agent RL problems such as zero-sum Markov games~\citep{daskalakis-foster-golowich20indep,sayin-et-al21decentralized,chen-et-al23finite-sample-zs} as well as for unconstrained MPGs in a recent line of works~\citep{leonardos_global_2021,zhang-et-al22softmax,zhang-et-al21grad-play,ding_independent_2022,maheshwari_independent_2023}. 

However, for \textit{constrained} MPGs, existing algorithms with convergence guarantees require coordination between players.
Indeed, inspired by~\citet{song-mei-bai22}, \citet{alatur_provably_2023} recently proposed a coordinate ascent algorithm for CMPGs in which each agent updates their policy in turn.  
At each time step, the policies of other agents are fixed while the updating agent faces a constrained Markov Decision Process (CMDP) to solve. 
When this coordination is not possible as in the independent learning protocol, the problem becomes more challenging as the environment is no longer stationary from the viewpoint of each agent and the problem does not reduce to solving a CMDP at each time step. 
This motivates the following question: 
\begin{quote}
\textit{Can we design an \textbf{independent} learning algorithm for  \textbf{constrained} MPGs with non-asymptotic global convergence guarantees?}
\end{quote}
 
\newpage
In this paper, we answer this question in the affirmative. 
Our contributions are as follows: 
\begin{itemize}
\item We design an algorithm for independent learning of constrained $\epsilon$-approximate Nash equilibria (NE) in CMPGs.
Inspired by recent works in nonconvex optimization under nonconvex constraints, our algorithm implements an inexact proximal-point update augmented with a regularized constraint set. In particular, the inexact proximal step is computed using a stochastic gradient switching algorithm for solving the resulting subproblem where both the objective and the constraint functions are strongly convex. Notably, the algorithm can be run independently by the different agents without taking turns. 
\item We analyze the proposed algorithm and establish its sample complexity to converge to an $\epsilon$-~approximate NE of the CMPG with polynomial dependence on problem parameters. Our analysis requires new technical developments that do not rely on results from the CMDP literature. 
\item We illustrate the performance of our algorithm on two simple CMPG applications: a pollution tax model and a marketplace for distributed energy resources. 
\end{itemize}

\begin{table}[h]
\caption{\label{tbl:contribution-grid} Position of our work in the literature. `centralized' means that the algorithm requires coordination between players who take turns in updating their policy; for `independent' learning, see~\cref{sec:preliminaries}.}
\begin{center}
\begin{tabular}{ccc}
\toprule
 & \makecell{centralized} & \makecell{independent} \\ 
 \midrule
MPG & \makecell{Nash-CA \\ \citet{song-mei-bai22}} & \makecell{Independent PGA \\ \citet{leonardos_global_2021}\\\citet{zhang-et-al21grad-play}\\\citet{ding_independent_2022}} \\ \hline
CMPG & \makecell{CA-CMPG \\ \citet{alatur_provably_2023}} & \makecell{\cref{alg:iprox-cmpg}\\\textit{This work}}\\
\bottomrule
\end{tabular}
\end{center}
\end{table}

\paragraph{Related Works}
We refer the reader to~\cref{tbl:contribution-grid} for a schematic positioning of our work in the recent literature. We next discuss some closely related work. 

\paragraph{Markov Potential Games} 
MPGs have been introduced as a natural extension of normal form potential games~\citep{monderer-shapley96} to the dynamic setting starting with state-based potential games~\citep{marden-et-al12}
and later Markov games~\citep{macua-et-al18}. 
\citet{leonardos_global_2021} introduced a variant of MPGs and proposed independent stochastic policy gradient methods with an~$\mathcal{O}(\epsilon^{-6})$ sample complexity to reach an $\epsilon$-approximate NE. Similar results were shown in~\citet{zhang-et-al21grad-play} with model-based algorithms. 
This result was later improved to an~$O(\epsilon^{-5})$ sample complexity for large state-action spaces with linear function approximation~\citep{ding_independent_2022} and further to an~$O(\epsilon^{-4.5})$ by reducing the variance of the agent-wise stochastic policy gradients~\citep{mao-et-al22}. \citet{zhang-et-al22softmax} explored the use of the softmax policy parametrization instead of the direct parametrization. In particular, they established an~$\mathcal{O}(\epsilon^{-2})$ iteration complexity in the deterministic setting and showed the benefits of using regularization to improve the convergence rate. \citet{maheshwari_independent_2023} proposed a fully independent and decentralized two timescale algorithm for MPGs with asymptotic guarantees where players may not even know the existence of other players.   
\citet{narasimha2022multi} provided verifiable structural assumptions under which a Markov game is an MPG and further provided several algorithms for solving MPGs in the deterministic setting.  
\citet{song-mei-bai22} proposed an~$\mathcal{O}(\epsilon^{-3})$ sample complexity coordinate ascent algorithm (Nash-CA) which requires coordination between players. 
\citet{guo-et-al23alpha-mpgs} recently introduced the class of~$\alpha$-MPGs which relaxes the definition of MPGs by allowing $\alpha$-deviations with respect to (w.r.t.) the potential function. More recently, \citet{zhou-et-al23networked-mpg} introduced a class of networked MPGs for which they proposed a localized actor-critic algorithm with linear function approximation. 
All the aforementioned works focused on the unconstrained setting.

\paragraph{Constrained Markov Games and CMPGs} 
There has been a vast array of works in multi-agent RL with safety constraints in practice (see e.g., \citet{elsayedaly-et-al21aamas-safe-marl,gu-et-al23} and the references therein). \citet{altman_constrained_2000} defined constrained Markov games and provided sufficient conditions for the existence of stationary constrained NE. Nonasymptotic theoretical convergence guarantees to game theoretic solution concepts for constrained multi-agent RL are relatively scarce in the literature. \citet{chen-et-al22constr-markov-games} introduced a notion of correlated equilibria for general constrained Markov games and provided a primal-dual algorithm for learning those equilibria.
\citet{ding-et-al23l4dc} established regret guarantees for episodic two-player zero-sum constrained Markov games. \citet{alatur_provably_2023} introduced the class of constrained MPGs. Inspired by Nash-CA~\citep{song-mei-bai22}, they proposed a constrained variant of the algorithm which enjoys an~$\mathcal{O}(\epsilon^{-5})$ sample complexity. Crucially, this algorithm requires coordination between agents and cannot be implemented independently by the agents. 

\paragraph{Inexact Proximal-Point}
The idea of using inexact proximal-point methods to solve \textit{nonconvex} problems has been fruitfully exploited in the literature for a couple of decades (see e.g., \cite{hare-sagastizabal09,davis-grimmer19}). A recent line of works~(\citet{boob-et-al23mathprog,ma_quadratically_2023}; and also~\citet{jia_first-order_2023}) extended this idea in order to solve nonconvex optimization problems with nonconvex functional constraints. The initial nonconvex problem is transformed into a sequence of convex problems by adding quadratic regularization terms to \textit{both} the objective and constraints. These works also established convergence rates to Karush–Kuhn–Tucker~(KKT) points under constraint qualification conditions. Our present work is inspired by this recent line of research. We point out though that we deal with a multi-agent RL problem and we provide convergence guarantees to approximate constrained NE. In these regards, our independent algorithm design and our analysis require several new technical developments.

\section{PRELIMINARIES}
\label{sec:preliminaries}

We consider an $m$-player constrained Markov Game where the players repeatedly select actions for maximizing their individual value functions while satisfying some constraints defined as cost value function bounds. More formally, the tabular game with random stopping, which we focus on, is described by a tuple~$\mathcal{G} = (\mathcal S, \mathcal N, \{ \mathcal A_i, r_i, c_i \}_{i \in \mathcal N}, \alpha, \mu, P, \kappa)\,$ with: 
\begin{itemize}
\item A finite set of states~$\mathcal{S}$ of cardinality~$S:=|\mathcal{S}|$ and a finite set of~$m$ agents~$\mathcal N := \{1, \dots, m\}\,.$
\item A finite set of actions~$\mathcal{A}_i$ of cardinality~$A_i:=|\mathcal{A}_i|$ for all~$i \in \mathcal N$ with $A_{\max}:=\max_{i \in \mathcal N}A_i$. The joint action space is denoted by~$\mathcal{A} := \prod_{i \in \mathcal N} \mathcal{A}_i\,.$
\item A reward function~$r_i: \mathcal{S} \times \mathcal{A} \to [0,1]$ and a cost function~$c_i: \mathcal{S} \times \mathcal{A} \to [0,1]$ for each agent~$i \in \mathcal N$. 
Throughout this paper, we will suppose that all the cost functions are identical across the agents and equal to a single cost function~$c$.\footnote{The case of multiple such common costs can be addressed with our approach with minor modifications. The case where cost functions may differ between players is more challenging and left for future work. See \cref{rem:other-constraints} for details.}

\item A distribution~$\mu$ over states from which the initial state of the game is drawn.
\item A probability transition kernel~$P$: For any state~$s \in \mathcal{S}$ and any joint action~$a \in \mathcal{A}$, the game transitions from state~$s$ to a state~$s' \in \mathcal{S}$ with probability~$P(s'|s, a)$ and the game terminates with probability~$\kappa_{s,a} > 0$. We further define~$\kappa := \min_{s \in \mathcal{S}, a \in \mathcal{A}} \kappa_{s,a}$ and $\gamma:=1-\kappa\,.$
\end{itemize}
At each time step~$t \geq 0$ of a given episode of the game, all the agents observe a shared state~$s_t \in \mathcal{S}$ and choose a joint action~$a_t \in \mathcal{A}$. Then, each agent~$i \in \mathcal N$ receives a reward~$r_i(s_t, a_t)$ and incurs a cost~$c(s_t, a_t)\,.$ The game either stops at time~$t$ with probability~$\kappa_{s_t, a_t}$ or proceeds by transitioning to a state~$s_{t+1}$ drawn from the distribution~$P(\cdot|s_t,a_t)\,.$ We denote by~$T_e$ the random stopping time when the episode terminates.\footnote{The discounted infinite horizon setting can also be addressed with minor adaptations.} 
For a similar setting in the unconstrained case, see~\citet{daskalakis-foster-golowich20indep,giannou-et-al22}.

\paragraph{Policies and Value Functions} Each agent~$i \in \mathcal{N}$ chooses their actions according to a randomized stationary policy denoted by $\pi_i \in \Pi^i:=\Delta \left( \mathcal A_i \right)^{\mathcal S}$ where~$\Delta \left( \mathcal A_i \right)$ is the probability simplex over the finite action space~$\mathcal{A}_i$. 
The set of joint policies~$\pi = \left( \pi_i \right)_{i \in \mathcal N}$ is denoted by~$\Pi := \times_{i \in \mathcal{N}} \Pi^i$ and we further use the notation $\pi_{-i}=\left( \pi_j \right)_{j \in \mathcal{N} \setminus \{i\}} \in \Pi^{-i} := \times_{j \in \mathcal{N} \setminus \{i\}} \Pi^j$ for joint policies of all agents other than~$i$. For any~$u \in \left\{ r_i \mid i \in \mathcal{N} \right\} \cup \left\{ c \right\}$ and any joint policy~$\pi \in \Pi$, we define the value function~$V_u(\pi)$ for every state~$s \in \mathcal{S}$ by~$V_{u,s}(\pi) := \E[ \sum_{t=0}^{T_e} u(s_t, a_t) | s_0=s].$ The shorthand notation~$V_u(\pi)$ will stand for~$V_{u,\mu}(\pi) := \E_{s \sim \mu}[V_{u,s}(\pi)]\,.$ For any policy $\pi \in \Pi$ and $s,s' \in \mathcal S$, the state visitation distribution is defined by $d_{s}^{\pi}(s'):= \mathbb{E}[\sum_{t=0}^{T_e} \1_{\{s_t=s'\}}|s_0=s]$ where~$\1$ is the indicator function and we write $d_{\mu}^{\pi}(s')=\E_{s \sim \mu}[d_{s}^{\pi}(s')]$.

In the rest of this paper, we will aim to minimize both rewards and costs to align with conventions from the constrained optimization literature. 
The equivalence to the common RL reward maximization formulation follows from considering reward functions~$1-r_i$ instead of~$r_i$ for each~$i \in \mathcal{N}\,.$ 

\paragraph{Constrained MPGs} In this paper, we consider an $m$-player constrained MPG (CMPG)~\citep{alatur_provably_2023} which is a constrained version of a Markov Potential Game~\citep{macua-et-al18,leonardos_global_2021}. In an MPG, for each state~$s \in \mathcal{S}$, there exists a so-called potential function~$\Phi_s: \Pi \to \R$ such that for all $i \in \mathcal N$, it holds that~$V_{r_i,s}(\pi_i,\pi_{-i})-V_{r_i,s}(\pi'_i,\pi_{-i}) = \Phi_s(\pi_i,\pi_{-i})-\Phi_s(\pi'_i,\pi_{-i})$ for any policies $\left( \pi_i,\pi_{-i} \right) \in \Pi$, and $\pi'_i \in \Pi^{i}\,.$ 
We will also use the notation~$\Phi(\pi) := \E_{s \sim \mu}[\Phi_s(\pi)]$.
Notice that the fully cooperative setting when all the reward functions of the players are identical is a particular instance of an MPG. 
Note also that the potential function is typically unknown for the players interacting in the game. 
The joint policies of the agents are constrained to the set~$\Pi_c:=\left\{ \pi \in \Pi \mid V_c(\pi) \leq \alpha \right\}$  of feasible policies. The set of feasible policies for agent~$i \in \mathcal N$ when the policy of the other agents is fixed to~$\pi_{-i} \in \Pi^{-i}$ is denoted by~$\Pi^i_c(\pi_{-i}):=\left\{ \pi_i \in \Pi^i \mid \left( \pi_i,\pi_{-i} \right) \in \Pi_c \right\}$.

\paragraph{Nash Equilibria} For any~$\epsilon \geq 0$, a joint policy~$\pi^{*} \in \Pi_c$ is called an $\epsilon$-approximate constrained NE if for every~$i \in \mathcal{N}$ and any policy~$\pi'_i \in \Pi^i_c(\pi^*_{-i}),$ we have $V_{r_i}(\pi^*)-V_{r_i}(\pi'_i,\pi^*_{-i}) \leq \epsilon$\,. When~$\epsilon = 0$, such a policy~$\pi^*$ is called a constrained NE policy and no agent has an incentive to deviate unilaterally from a NE policy~$\pi^*$.
Observe that unilateral deviations are only allowed within the set of feasible policies in our constrained setting. We refer the reader to~\citet{altman_constrained_2000} for the existence of stationary constrained NEs.

\paragraph{Independent Learning Protocol} All the players interact via executing their policies for a fixed number of episodes in order to find an approximate constrained NE. Importantly, during the learning procedure, each player executes their policy at each episode of the game to sample a trajectory and exclusively observes their own trajectory~$(s_t,a_{i,t}, r_i(s_t,a_t), c(s_t,a_t))_{0 \leq t \leq T_e}\,.$ In particular, a player does not have access to the policies of other players or their chosen actions.
Such a protocol was considered for instance in two-player zero-sum Markov games in~\citet{daskalakis-foster-golowich20indep,chen-et-al23finite-sample-zs} as well as for unconstrained MPGs~\citep{leonardos_global_2021, ding_independent_2022,maheshwari_independent_2023}.

\section{INDEPENDENT ALGORITHM FOR CONSTRAINED MPGs}
\label{sec:ind-learning-algo}

In this section, we present our independent \textsf{iProxCMPG} algorithm for learning constrained NE in CMPGs.
Before describing our approach, we discuss an alternative, natural but unsuccessful, approach to motivate our algorithm design. This will allow us to highlight the challenges arising from the combination of (a) the presence of coupled constraints, (b) the multi-player setting, and (c) the independent learning protocol.

Our starting point is the known result that any maximizer of the potential function is a NE of the game. This result was initially proved by~\citet{monderer-shapley96} for normal form potential games and later generalized to MPGs by~\citet{leonardos_global_2021} and to constrained MPGs more recently~\citep{alatur_provably_2023}. Therefore, in order to find an (approximate) constrained NE for our CMPG\footnote{Approximate KKT points of this problem will be related to approximate constrained NE of our CMPG.}, we will consider solving the following constrained optimization problem: 
\begin{equation}
\label{eq:pb-phi-constrained}
\min_{\pi \in \Pi_c} \Phi(\pi)\,,
\end{equation}
where~$\Phi$ is the potential function for our CMPG using the notations introduced in section~\ref{sec:preliminaries}. 
This problem involves a nonconvex objective with a nonconvex constraint since the value function is a nonconvex function of the policy in general (see e.g., Lemma~1 in~\citet{agarwal-et-al21pg}). However, although nonconvex optimization problems with nonconvex constraints are notoriously hard, it turns out that problem~\cref{eq:pb-phi-constrained} is still tractable in the single agent setting. 
In this case, the problem boils down to a CMDP problem. Despite its nonconvexity, the problem can be recast as a linear program in the space of occupancy measures which is a convex set (see Chapter~3 in~\citet{altman99cmdps}). 
Then, strong duality permits to design primal-dual policy gradient algorithms to solve the problem with convergence guarantees (see e.g., \citet{paternain_constrained_2019}).  

Given those positive results for single agent CMDPs, a natural approach is to derive a primal-dual algorithm for our multi-agent problem~\cref{eq:pb-phi-constrained} as it was proposed by~\citet{diddigi_actor-critic_2020}. In the latter work, a primal-dual policy gradient algorithm was proposed using the Lagrangian function~$\mathcal L(\pi,\lambda):=\Phi(\pi)+\lambda(V_c(\pi)-\alpha)$\, where~$\lambda \geq 0$ is a Lagrange multiplier. This algorithm can then be run independently by the different agents using existing independent learning algorithms for the unconstrained setting~\citep{leonardos_global_2021,zhang-et-al21grad-play,ding_independent_2022}. Unfortunately, it has been recently shown by~\citet{alatur_provably_2023} that strong duality does not hold in general for the CMPG problem. As a consequence, it is not clear how to obtain guarantees for convergence to constrained NE using this duality approach. This is due to the multi-agent nature of our problem. 
In particular, since the constraint couples the agents' individual policies, the set of state-action occupancy measures induced by joint policies of the players cannot be obviously split into several convex problems involving the occupancy measures induced by each one of the players' policies. The well-known challenge of nonstationarity of the environment in multi-agent RL makes the design of independent learning algorithms difficult. 
As a remedy, \citet{alatur_provably_2023} resort to coordination among players and propose a coordinate ascent algorithm for CMDPs. 
At each time step and for every player~$i$, by fixing the policy of other players but player~$i$ to $\pi_{-i}$, player $i$ can learn a ``best-response'' policy by solving a CMDP since the environment now becomes stationary from agent~$i$'s viewpoint. 

Recall now that our main objective is to design an \textit{independent} learning algorithm in the sense of section~\ref{sec:preliminaries} in order to learn constrained NE for our CMPG. 
We now describe our approach which takes a different route. 
Our algorithm is inspired by recent work in nonconvex optimization under nonconvex constraints~\citep{boob-et-al23mathprog,ma_quadratically_2023,jia_first-order_2023}. 
Following their ideas, we consider the following proximal update with penalized constraints: 
\begin{equation}
\label{eq:prox-update}
\begin{aligned}
\pi^{(t+1)} = \argmin_{\pi \in \Pi} \Big\{ &\Phi \left( \pi \right) + \frac{1}{2\eta} \left\| \pi-\pi^{(t)} \right\|^2 \Big| \\ &V_c(\pi) + \frac{1}{2\eta} \left\| \pi-\pi^{(t)} \right\|^2 + 
\beta \leq \alpha \Big\}
\end{aligned}
\end{equation}
where~$\pi^{(0)}$ is a given initial joint policy, $\eta > 0$ is a step size and~$\beta>0$ an additional slack.
Observe that~$V_c(\pi^{(t+1)}) + \beta - \alpha \leq - \|\pi^{(t+1)} - \pi^{(t)}\|^2/2\eta\,.$ Hence, the policy~$\pi^{(t)}$ is feasible with slack $\beta$, i.e., $V_c(\pi^{(t)}) + \beta \leq \alpha$, for every~$t \geq 0$.
We introduce two additional notations for convenience. Define for any joint policies~$\pi, \pi' \in \Pi,$ and~$\eta >0$, 
\begin{align*}
\Phi_{\eta,\pi'}(\pi) &:=\Phi(\pi) + \frac{1}{2\eta} \left\| \pi-\pi' \right\|^2,\\
V_{\eta,\pi'}^c(\pi)&:=V_c(\pi)+\frac{1}{2\eta} \left\| \pi-\pi' \right\|^2,\\
\Pi^c_{\eta,\pi'}&:=\left\{ \pi \in \Pi \mid V^c_{\eta,\pi'}(\pi) + \beta \leq \alpha \right\}\,.
\end{align*}
Our update rule in~\cref{eq:prox-update} can then be rewritten as: 
\begin{align}
\label{eq:prox-update-shorthand}
\pi^{(t+1)} = \argmin_{\pi \in \Pi_{\eta, \pi^{(t)}}^{c}} \Phi_{\eta, \pi^{(t)}}(\pi)\,.
\end{align}
We immediately observe that the above update rule is well-defined since~$\Phi_{\eta, \pi^{(t)}}$ and~$V^{c}_{\eta, \pi^{(t)}}$ are strongly convex for every~$t \geq 0$ for a suitable step size~$\eta$. This is in contrast with the original problem where both the potential function~$\Phi$ and the constraint function~$V_c$ are smooth but nonconvex. 
We also remark that if $\pi^{(t)}$ converges, then the regularization term~$\left\| \pi^{(t+1)}-\pi^{(t)} \right\|$ becomes small and the surrogate feasible region~$\Pi^c_{\eta, \pi^{(t)}}$ approaches the original constraint set~$\Pi_c$ up to the additional slack $\beta$. 

Now, we discuss how to solve the proximal problem in~\cref{eq:prox-update-shorthand} defining our main update rule. To solve this strongly convex problem with strongly convex constraint, we adapt a gradient switching algorithm proposed in~\citet{lan-zhou20csa}. At each iteration~$k$, our algorithm performs a projected gradient descent step along either the gradient of the (regularized) objective or the gradient of the constraint function depending on whether an estimate of the constraint function satisfies the relaxed constraint~$V_c(\pi^{(t,k)}) + \beta - \alpha \leq \delta_k$ where $(\delta_k)$ is a decreasing sequence converging to zero and hence progressively enforcing the constraint. However, it is not immediate from the above procedure how to obtain an independent learning algorithm specifying an update rule for each player without coordination between the players.
Recall for instance that the potential function~$\Phi$ is unknown to the players in general and full gradients of both the potential and constraint functions w.r.t.\ the joint policy cannot be available to each agent since we exclude coordination and centralization.  
To obtain our independent~\textsf{iProxCMPG}, see \cref{alg:iprox-cmpg}, we propose to use agent-wise updates where each agent runs the gradient switching algorithm independently using only partial gradients of the potential and constraint functions w.r.t.\ their individual policy. Notice that our subroutine algorithm deviates from the one proposed in~\citet{lan-zhou20csa} in that we use the estimate of the constraint function~$V_c$ \textit{instead} of the \textit{regularized} constraint function~$V^c_{\eta,\pi^{(t)}}$. This is because the regularized constraint function involves the \textit{joint policy} in the regularization while the constraint value function can be estimated independently.
We further remark that for our analysis, the index $\hat k$ sampled in line~7 of \cref{alg:iprox-cmpg} is supposed to be picked the same by all the players (see also Remark~\ref{remark:index-sampling} in Appendix~\ref{sec:appendix-full-stochastic-alg} for further details).  

\begin{algorithm*}
\caption{\textsf{iProxCMPG}: \textbf{i}ndependent \textbf{Prox}imal-policy algorithm for \textbf{CMPG}s}
\label{alg:iprox-cmpg}
\begin{algorithmic}[1]
\State \textbf{initialization:} $\pi^{(0)} \in \Pi^{\xi}$ s.t.\ $V_c(\pi^{(0)})<\alpha$ and suitably chosen $\eta,\beta,T,K,\{(\nu_k,\delta_k,\rho_k)\}_{0 \leq k \leq K}$
\For{$t=0,\dots,T-1$}
\State $\pi^{(t,0)}_i=\pi^{(t)}_i$ for $i \in \mathcal N$
\For{$k=0,\dots,K-1$ and $i \in \mathcal N$ simultaneously}
\State $\pi^{(t,k+1)}_i=
\begin{cases} 
\mathcal P_{\Pi^{i,\xi}} \left[ \pi^{(t,k)}_i - \nu_k \hat{\nabla}_{\pi_i}V_{r_i}(\pi^{(t,k)})-\frac{\nu_k}{\eta} (\pi^{(t,k)}_i-\pi_i^{(t)}) \right] &\text{if } \hat{V}_c(\pi^{(t,k)}) + \beta - \alpha \leq \delta_k \\[4pt] 
\mathcal P_{\Pi^{i,\xi}} \left[ \pi^{(t,k)}_i-\nu_k \hat{\nabla}_{\pi_i}V_{c}(\pi^{(t,k)})-\frac{\nu_k}{\eta} (\pi^{(t,k)}_i-\pi_i^{(t)}) \right] &\text{otherwise}\end{cases}$
\EndFor
\State $\mathcal B^{(t)}=\{ \lfloor K/2 \rfloor \leq k \leq K \mid \hat V_c(\pi^{(t,k)}) \leq \delta_k \}$
\State $\pi_i^{(t+1)}=\pi_i^{(t,\hat k)}$ where $\hat k=1$ if $\mathcal B^{(t)}=\emptyset$ and else sampled s.t.\ for $k \in \mathcal B^{(t)}$, $\mathbb P(\hat k = k)=\left( \sum_{k \in \mathcal B^{(t)}} \rho_k \right)^{-1} \rho_k$
\EndFor
\State \textbf{output:} $\pi_i^{(T)}$ for $i\in \mathcal N$
\end{algorithmic}
\end{algorithm*}

\paragraph{Stochastic Setting} When exact gradients and value functions are not available, we estimate them using sampled trajectories. For each joint policy~$\pi^{(t,k)}$, every player~$i$ samples a trajectory~$\tau_i := (s_j^{(t,k)}, a_{i,j}^{(t,k)}, r_{i,j}^{(t,k)}, c_j^{(t,k)})_{0 \leq j \leq T_e}$ of length~$T_e+1$ by executing their own policy~$\pi_i^{(t,k)}$. Here, $s_0^{(t,k)} \sim \mu$ and $r_{i,j}^{(t,k)}, c_j^{(t,k)}$ respectively refer to the reward and cost incurred by the $i$-th player at the $j$-th step. 
The gradients $\nabla_{\pi_i}V_{r_i}(\pi^{(t,k)})$ and $\nabla_{\pi_i}V_c(\pi^{(t,k)})$ are replaced by their sample estimates
\begin{equation}
\label{eq:stochastic-pgs}
\begin{aligned}
\hat\nabla V^{r_i}_{\pi_i}(\pi^{(t,k)}) &:= R_i^{(T_e,t,k)}\, \psi_{\pi_i^{(t,k)}}^{T_e}\,,\\
\hat\nabla V^c_{\pi_i}(\pi^{(t,k)}) &:= C^{(T_e,t,k)}\, \psi_{\pi_i^{(t,k)}}^{T_e}\,,
\end{aligned}
\end{equation}
where $R_i^{(T_e,t,k)}:=\sum_{j=0}^{T_e} r_{i,j}^{(t,k)}$, $C^{(T_e,t,k)}:=\sum_{j=0}^{T_e} c_j^{(t,k)}$ and $\psi_{\pi_i^{(t,k)}}^{T_e} :=  \sum_{j=0}^{T_e} \nabla_{\pi_i} \log \pi_i^{(t,k)} \left( a_{i,j}^{(t,k)} \mid s_j^{(t,k)} \right).$
Each agent estimates~$V_c(\pi^{(t,k)})$ by~$\hat V_c(\pi^{(t,k)}):=C^{(T_e,t,k)}$ independently, using the cost feedback information they receive. Note that details of trajectory sampling are omitted in \cref{alg:iprox-cmpg} for more compact presentation (for the full version, see \cref{alg:iprox-cmpg-sto} in \cref{sec:appendix-full-stochastic-alg}).

\begin{remark}
\label{rem:other-constraints}
As potential avenues for future work, we would like to point out two possible generalizations of the considered CMPG setting in which our current \textsf{iProx-CMPG} algorithm and analysis are not directly applicable:
\begin{itemize}
\item \textbf{Potential cost constraints.} Suppose we do not require the cost functions $c_i$ to be identical across all players $i \in \mathcal N$ but instead assume that for each~$s \in \mathcal S$, there exists a so-called cost potential function $\Phi_{c,s}:\Pi \to \R$ such that for all $i \in \mathcal N$, it holds that~$V_{c_i,s}(\pi_i,\pi_{-i})-V_{c_i,s}(\pi'_i,\pi_{-i}) = \Phi_{c,s}(\pi_i,\pi_{-i})-\Phi_{c,s}(\pi'_i,\pi_{-i})$ for any policies~$\left( \pi_i,\pi_{-i} \right) \in \Pi$, and $\pi'_i \in \Pi^{i}\,.$\\
\emph{Note that in order to use the gradient switching subroutine in our algorithm, it is essential that all agents are able to estimate whether or not the constraint holds for the current joint policy. In the case of a potential cost, it is not clear how to provide such estimates unless agents have knowledge of the potential (e.g.\ as a known function of the cost). This is an interesting question that merits further investigation.}

\item \textbf{Playerwise cost thresholds.} Suppose each player $i \in \mathcal N$ has an individual feasibility threshold~$\alpha_i$. The set of feasible policies is then redefined as $\Phi_c := \left\{\pi \in \Pi \mid \forall i \in \mathcal N, V_{c_i}(\pi) \leq \alpha_i\right\}.$\\
\emph{In this case, if all the cost functions $c_i$ are identical \emph{and} if the thresholds $\alpha_i$ can be communicated among players, one can consider the hardest constraint $\alpha := \min_{i \in \mathcal N}\alpha_i$ and use our approach to find a policy solving this stricter problem (if such policy exists). Otherwise, if cost functions are not identical or if each agent has their private threshold, our algorithm and analysis need further adjustments.}
\end{itemize}
\end{remark}

\section{CONVERGENCE ANALYSIS AND SAMPLE COMPLEXITY}\label{sec:analysis-complexity}
In this section, we establish the iteration complexity of~\cref{alg:iprox-cmpg} in the deterministic setting before stating its sample complexity in the stochastic setting.
We first introduce our assumptions. The first one guarantees the existence of a strictly feasible policy that is available to the agents for initialization.
\begin{assumption}
\label{ass:init-feasible}
The initial policy~$\pi^{(0)}$ satisfies~$V_c(\pi^{(0)}) < \alpha\,.$
\end{assumption}

A few remarks are in order regarding this assumption: 
\begin{itemize}
\item Similar assumptions have been made in the related constrained optimization literature when dealing with nonconvex constraints~\citep{boob-et-al23mathprog,ma_quadratically_2023,jia_first-order_2023}. Otherwise, satisfying a constraint may require finding a global minimizer which is computationally intractable in a general nonconvex setting. In our case, this corresponds to finding the global minimizer of a potential function in a fully cooperative unconstrained MPG. While this can be achieved in a single agent setting thanks to the gradient dominance property~\citep{agarwal-et-al21pg,xiao22}, such a global optimality result is not available in the literature for our multi-agent setting to the best of our knowledge.
\item While finding a strictly feasible policy is involved in general, it may be possible to find such a policy in some special cases, such as when the state space can be factored, the probability transitions are independent across agents and the constraint cost functions are separable (see examples~1 and~2 in~\citet{alatur_provably_2023} for more details).
\end{itemize}

In addition to initial feasibility, we require that Slater's condition holds for each subproblem given by a proximal-point update. 
This is ensured by the following uniform Slater's condition. 
\begin{assumption}
\label{ass:uniform-slater}
Let $\eta=\frac{1}{2L_\Phi}$ where $L_\Phi$ is the smoothness parameter\footnote{See also \cref{lem:basics}, item~\ref{item:unregularized-smoothness} for an expression of $L_\Phi$ in terms of $m,\gamma$, and $A_{\max}.$} of $\Phi$. Then, there exists $\zeta>0$ such that for any strictly feasible $\pi' \in \Pi$, i.e., $V_c(\pi')<\alpha$, there exists $\pi \in \Pi$ with~$V^c_{\eta,\pi'}(\pi) \leq \alpha -\zeta$.
\end{assumption}
We make the following comments:
\begin{itemize}
\item First, we point out that a strictly feasible $\pi'$ satisfies $V^c_{\eta,\pi'}(\pi') = V_c(\pi') < \alpha$, i.e., existence of a strictly feasible policy for the regularized constraint function $V^c_{\eta,\pi'}$ is trivially given. \cref{ass:uniform-slater} additionally ensures that strict feasibility holds with slack $\zeta$ where $\zeta$ is independent of $\pi'$.
\item  Similar constraint qualification conditions have been widely used in the nonconvex constrained optimization literature, see~\citet{boob2022level}, Table~1 for an overview. In particular, \cref{ass:uniform-slater} is similar to the uniform Slater's condition of~\citet{ma_quadratically_2023}. Assumption 3 in \citet{boob-et-al23mathprog} is a strong feasibility assumption which implies \cref{ass:uniform-slater}, and hence could also replace it here. Strong feasibility assumes existence of a policy $\pi$ such that $V_c(\pi) \leq \alpha-\frac{\text{diam}(\Pi)^2}{\eta}$ where $\text{diam}(\Pi):=\max_{\pi,\pi' \in \Pi}\| \pi-\pi' \|$.
\item A uniform strict feasibility assumption similar to \cref{ass:uniform-slater} was used for centralized NE-learning, see~\citet{alatur_provably_2023}, Assumption 2.
\end{itemize}

\paragraph{Exact Gradients Case}
In the noiseless setting with access to exact gradients, we achieve the following iteration complexity result.
\begin{theorem}
\label{thm:main-exact-gradients}
Let \cref{ass:init-feasible,ass:uniform-slater} hold and let the distribution mismatch coefficient $D := \max_{\pi \in \Pi} \left\| d_{\mu}^{\pi}/\mu \right\|_{\infty}$ be finite. For any $\epsilon>0$, after running \textsf{iProxCMPG}, \cref{alg:iprox-cmpg}, for $\xi=0$, suitably chosen~$\eta,\beta,T,K$, and~$\{(\nu_k,\delta_k,\rho_k)\}_{0 \leq k \leq K}$, there exists~$t \in [T]$, such that $\pi^{(t)}$ is a constrained $\epsilon$-NE\@. The total iteration complexity is given by~$\mathcal O(\epsilon^{-4})$ where $\mathcal O(\cdot)$ hides polynomial dependencies in $m,S,A_{\max},D, 1-\gamma$, and~$\zeta$.
\end{theorem}
The full proof of \cref{thm:main-exact-gradients} is deferred to \cref{sec:appendix-proof-exact-gradients}. We briefly outline the key steps below.
\begin{proof}[Proof idea]
First, we show that $K=\mathcal{O}(\epsilon^{-2})$ iterations of the inner loop yield a policy that is feasible and achieves potential value sufficiently close to the exact  proximal update \cref{eq:prox-update-shorthand}. For $T=\mathcal{O}(\epsilon^{-2})$, standard arguments then imply existence of $t \in [T]$ such that $\| \pi^{(t+1)}-\pi^{(t)} \| = \mathcal{O}(\epsilon)$. It can further be shown that such $\pi^{(t+1)}$ satisfies a particular form of approximate CMPG-specific KKT conditions for the original constrained optimization problem \cref{eq:pb-phi-constrained}. We then leverage the multi-agent structure to argue that for all $i \in \mathcal N$, similar KKT conditions also hold w.r.t.\ the playerwise problem $\min_{\pi_i \in \Pi_c^i(\pi_{-i}^{(t+1)})} V_{r_i}(\pi_i,\pi_{-i}^{(t+1)})$ where $\pi_{-i}^{(t+1)}$ is fixed.
Finally, using playerwise gradient dominance (see e.g., Lemma~D.3 in~\citet{leonardos_global_2021} or Lemma~2 in~\citet{giannou-et-al22}), one can bound the duality gap of player $i$'s constrained problem for all~$i \in \mathcal N$ which implies that $\pi^{(t+1)}$ is a constrained $\epsilon$-NE. The total iteration complexity is given by $T \cdot K = \mathcal{O}(\epsilon^{-4})$.
\end{proof}

\paragraph{Finite Sample Case} 
In the stochastic setting, when exact gradients are not available, the variance of the stochastic policy gradients in~(\ref{eq:stochastic-pgs}) can be unbounded if the policies get closer to the boundaries of the simplex (see e.g., Eq.~(13) in~\cite{giannou-et-al22}).  
Therefore, we consider exploratory $\xi$-greedy policies to address this issue as in prior work \citep{daskalakis-foster-golowich20indep,leonardos_global_2021,ding_independent_2022,giannou-et-al22}.  
Define for any $\xi \geq 0, i \in \mathcal{N}$ the subset of $\xi$-greedy policies
\begin{align*}
\Pi^{i,\xi} := \left\{ \pi \in \Pi \mid \forall s \in \mathcal S: \pi_i \left( \cdot \mid s \right) \geq \xi/A_i \right\}\,, 
\end{align*}
which is used in Algorithm~\ref{alg:iprox-cmpg}. 
We are now ready to state our sample complexity result. 

\begin{theorem}
\label{thm:main-finite-sample}
Let \cref{ass:init-feasible,ass:uniform-slater} hold, and let $D$ (as in \cref{thm:main-exact-gradients}) be finite. Then, for any $\epsilon > 0$, after running \textsf{iProxCMPG} based on finite sample estimates (see \cref{alg:iprox-cmpg-sto}) for suitably chosen $\eta,\beta,\xi,T,K,B$, and $\{(\nu_k,\delta_k,\rho_k)\}_{0 \leq k \leq K}$, there exists $t \in [T]$, such that in expectation, $\pi^{(t)}$ is a constrained $\epsilon$-NE\@. The total sample complexity is given by~$\tilde{\mathcal{O}}(\epsilon^{-7})$ where $\tilde{\mathcal{O}}(\cdot)$ hides polynomial dependencies in $m,S,A_{\max},D,1-\gamma$, and~$\zeta$, as well as logarithmic dependencies in $1/\epsilon$.
\end{theorem}

We refer the reader to \cref{sec:appendix-proof-finite-sample} for the proof of \cref{thm:main-finite-sample}. Below, we briefly explain how we obtain our sample complexity result.
\begin{proof}[Proof idea]
As in the exact gradients case, we require $T=\mathcal{O}(\epsilon^{-2})$ iterations of the outer loop. In the stochastic setting, our independent implementation of the CSA algorithm~\citep{lan-zhou20csa} still converges at a $\mathcal{O}(1/K)$-rate due to strong convexity, but requires sampling a batch of size $B=\mathcal{O}(\epsilon^{-2})$ for estimating constraint function values at each iteration. To counteract the variance of $\xi$-greedy gradient estimates (which in our case grows as $\mathcal{O}(\epsilon^{-1})$), we need to set $K=\mathcal{O}(\epsilon^{-3})$. All in all, we end up with sample complexity $T \cdot K \cdot B = \mathcal{O}(\epsilon^{-7})$ for proving existence of $t \in [T]$ such that $\E \left[ \left\| \pi^{(t)}-\pi^{(t+1)} \right\| \right] = \mathcal{O}(\epsilon)$. Using similar arguments as for \cref{thm:main-exact-gradients}, this implies that~$\pi^{(t+1)}$ is a constrained $\epsilon$-NE in expectation.
\end{proof}
\begin{remark}
Comparing our result to the state-of-the-art in the unconstrained case ($\mathcal{O}(\epsilon^{-5})$, \citet{ding_independent_2022}), accounting for constraints comes at a cost, increasing the sample complexity by a $\mathcal{O}(\epsilon^{-2})$-factor. In the centralized setting, a similar gap can be observed between best known results for unconstrained ($\mathcal{O}(\epsilon^{-3})$, \citet{song-mei-bai22}) vs.\ constrained ($\mathcal{O}(\epsilon^{-5})$, \citet{alatur_provably_2023}) NE-learning. Whether this $\mathcal{O}(\epsilon^{-2})$-gap can be narrowed is an interesting open question for both centralized and independent learning. 
\end{remark}

\section{SIMULATIONS}
\label{sec:experiments}

We test our stochastic \textsf{iProxCMPG} algorithm in two simple applications that can be modeled as CMPGs and for which coordination among players is unrealistic. Both examples are inspired by unconstrained variants presented in~\citet{narasimha2022multi} who study MPGs. Our code is publicly available\footnote{\url{https://github.com/philip-jordan/iProx-CMPG}}.

\paragraph{Pollution Tax Model} Consider a simple environment with $m$ agents representing e.g.\ factories, two states, \emph{pollution-free} and \emph{polluted}, and two actions, \emph{clean} and \emph{dirty} corresponding to low and high production volume. Starting in the \emph{pollution-free} state, in each round, the environment transitions to the \emph{polluted} state if and only if at least one agent chooses \emph{dirty}. Each agent's reward is the sum of its profit minus a pollution tax. In either state, the profit is~$P_c$ when choosing \emph{clean} and~$P_d$ when choosing \emph{dirty}. The pollution tax is zero in the \emph{pollution-free}, and $T_p$ in the \emph{polluted} state. As pointed out by~\citet{narasimha2022multi}, due to rewards being separable in the sense that $r_i(s,a_i,a_{-i})=r'_i(s)+r''_i(a_i,a_{-i})$ and state transition probabilities being state independent, the pollution tax model satisfies a sufficient condition under which a Markov game is an MPG. For our simulations, we set $P_c=2, P_d=4$, and $T_p=4$.
Due to the lack of incentives for agents to cooperate when promoting environmental sustainability, requiring coordination is unrealistic in this example. Moreover, note that the purpose of the pollution tax is to counteract pollution by penalizing \emph{dirty} actions. However, in practice, there may be additional global requirements on the minimum total production volume. To model this as a CMPG, we charge a cost $C$ per agent that chooses \emph{clean} and impose the constraint $V_c(\pi) \leq \alpha_C$ for appropriately chosen $\alpha_C$.

\begin{figure}[h]
\vspace{.3in}
\centerline{\includegraphics[width=\columnwidth]{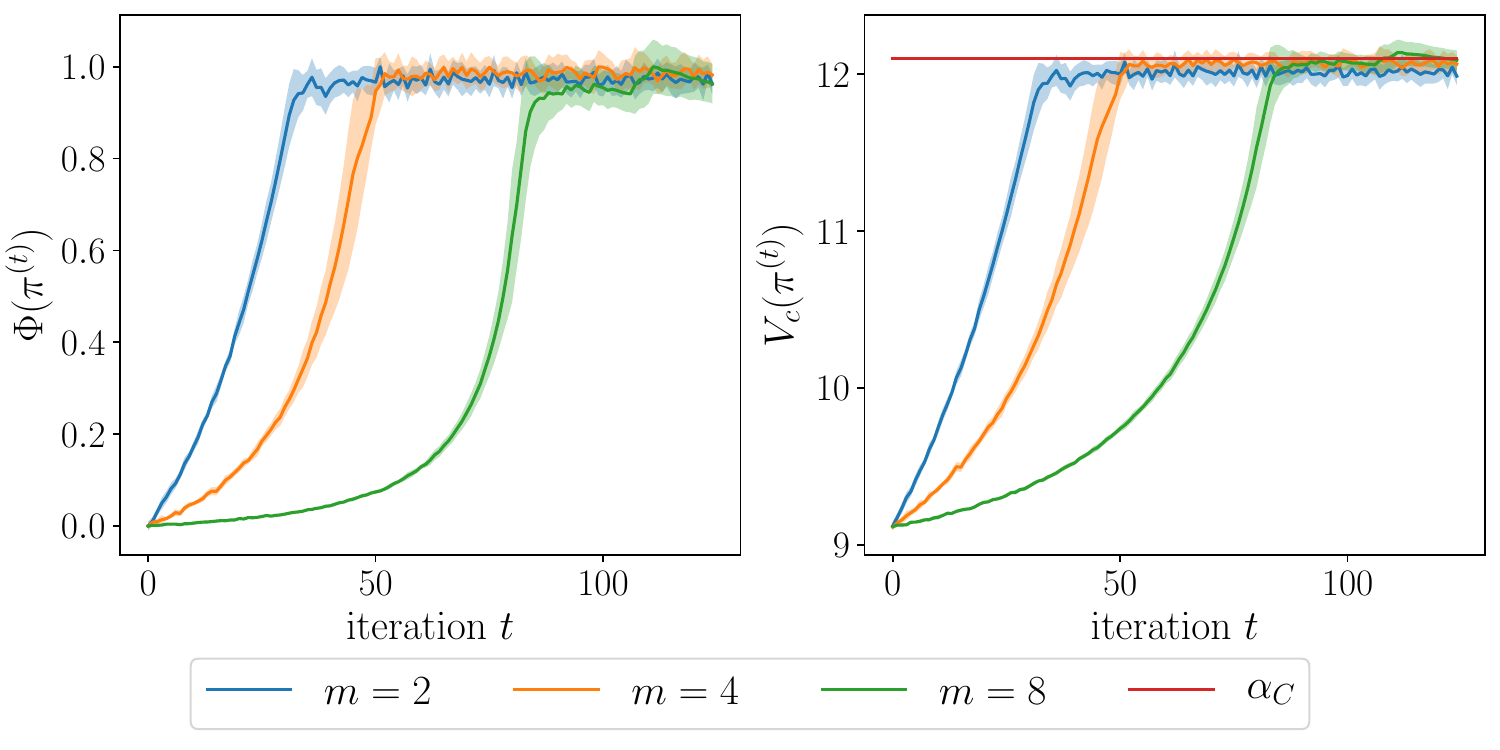}}
\vspace{.3in}
\caption{Potential (left, scaled to $[0,1]$) and constraint (right) values of \textsf{iProxCMPG} for the $m$-agent pollution tax model.}
\label{fig:pollution-sim-plot}
\end{figure}
We run \textsf{iProxCMPG} on the resulting $m$-agent CMPG for $m \in \{2,4,8\}$ and with $C=1, \alpha_C=12$. Hyperparameter choices are reported in \cref{sec:additional-experiments} and \cref{tab:params}. \cref{fig:pollution-sim-plot} shows the mean and standard deviation (shaded region) across independent runs of per-iteration potential and constraint values. Note that unlike in the theory part, we use the potential maximization perspective for experiments. We observe convergence to a constrained NE under which the minimum production requirements are approximately satisfied.

\paragraph{Marketplace for Distributed Energy Resources}
As more and more small-scale electricity producers enter the electrical grids, a marketplace emerges. Each participant needs to decide how much energy to sell given the current supply and demand. The competitive nature of such marketplaces motivates studying the convergence of independent algorithms to NEs under the constraints imposed by market rules.
The CMPG we consider has states $\mathcal S=\left\{ 0,\dots,S-1 \right\}$ indicating the grid's current energy demand from high at $0$ to low at $S-1$. Action $a_i \in \mathcal A_i=\left\{0,\dots,A_i-1\right\}$ represents the units of energy agent~$i$ contributes, for which it is rewarded with profit $r_i(s,a_i,a_{-i})=c_0 a_i^2 - c_1 a_i^2 \sum_{i \in \mathcal N}a_i - a_i c_2^s$ where $c_0,c_1,c_2$ are model parameters. State transitions are modeled by first sampling $w \sim \mathcal U(\{0,1,\dots,W\})$ which models uncertainty due to e.g.\ weather, and then setting $s'=\max\{0,\min\{S-1,\sum_{i \in \mathcal N}a_i - w\}\}$ with probability $0.9$ and $s'=w$ otherwise. For our simulations, we set $S=A=W=5, c_0=2, c_1=0.25$, and $c_2=1.25$. \citet{narasimha2022multi} show that the described game is indeed an MPG with $\Phi(\pi)=\E_{\pi,s_0 \sim \mu}[ \sum_{t=0}^{T_e} \phi_{s_t}(a_t)]$ and $\phi_s(a_i,a_{-i})=c_0\sum_{i \in \mathcal N}a_i-c_1\sum_{i \in \mathcal N}a_i^2-c_1\sum_{1 \leq i < j \leq m}a_i a_j - m c_2^s$. 

\begin{figure}[h]
\vspace{.3in}
\centerline{\includegraphics[width=\columnwidth]{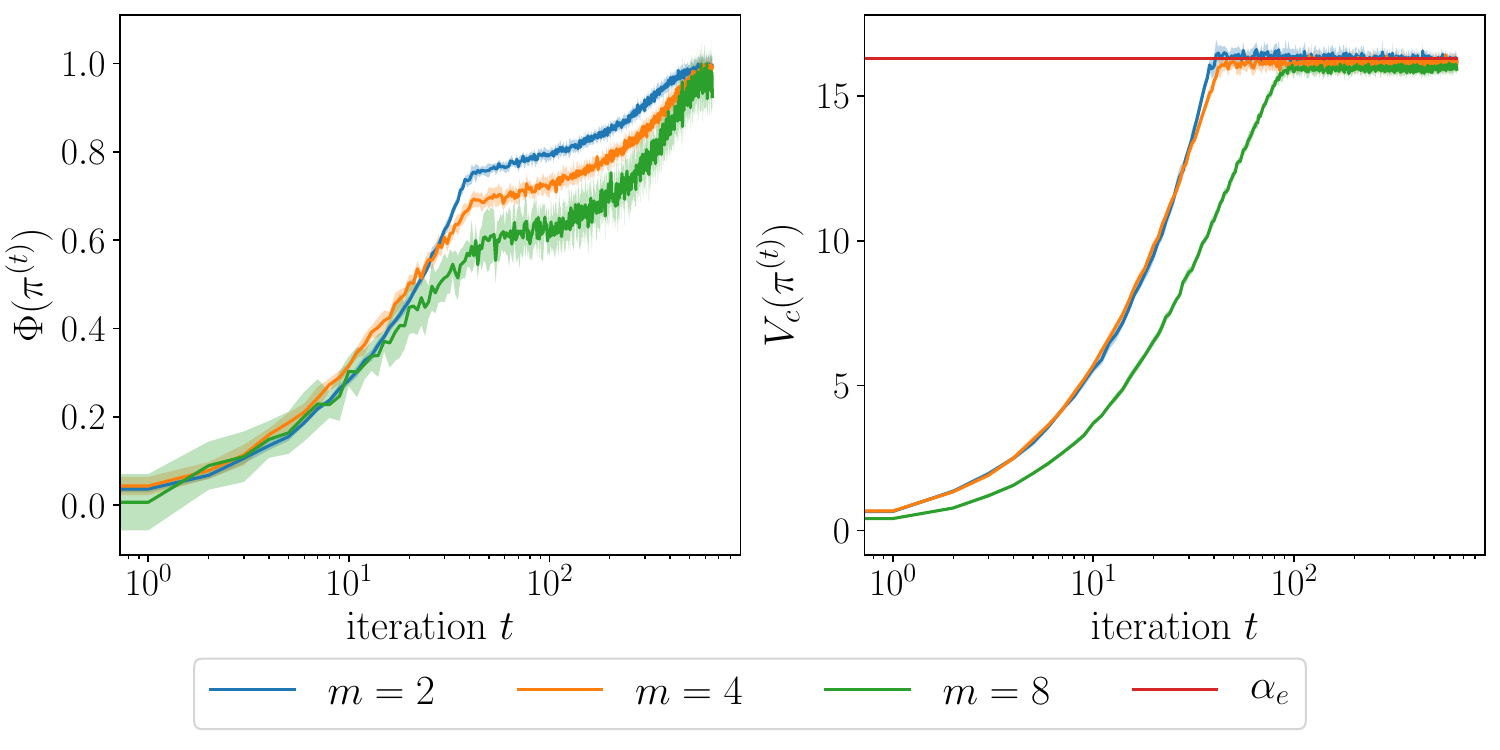}}
\vspace{.3in}
\caption{Potential (left) and constraint (right) values of \textsf{iProxCMPG} for the $m$-agent energy marketplace.}
\label{fig:der-sim-plot}
\end{figure}

We extend this game into a CMPG by having the system incur a cost per unit of energy provided to the grid, i.e., by defining $c(s,a)=\sum_{i \in \mathcal N}a_i$ for all $s \in \mathcal S$, and requiring $V_c(\pi) \leq \alpha_e$ where we set $\alpha_e=16$. \cref{fig:der-sim-plot} shows convergence to a constrained NE where players satisfy the energy provision bound on average. 

\section{CONCLUSION}
\label{sec:conclusion}

In this paper, we proposed an independent learning algorithm for learning constrained NE in CMPGs. 
Our work opens up a number of avenues for future work. 
It would be interesting to investigate whether our sample complexity can be improved to match the better sample complexity of centralized algorithms. 
Our algorithm and theoretical guarantees require the agents to run the same algorithm: This may be seen as implicit coordination between agents. Designing fully independent learning dynamics for our constrained setting, where the players may not even be aware of the existence of other players is an interesting direction.  Going beyond the class of CMPGs for learning constrained NE is another research direction that is worth exploring. 
Using function approximation to scale to large state-action spaces beyond the tabular setting is also a promising prospect for future work. 

\subsubsection*{Acknowledgements}
We would like to thank the anonymous reviewers for their valuable comments. The work is supported by ETH research grant,  Swiss National Science Foundation (SNSF) Project Funding No. 200021-207343, and SNSF Starting Grant.  A.B. acknowledges support from the ETH Foundations of Data Science (ETH-FDS) postdoctoral fellowship.

\bibliography{references}
\bibliographystyle{apalike}

\section*{Checklist}

\begin{enumerate}
\item For all models and algorithms presented, check if you include:
\begin{enumerate}
\item A clear description of the mathematical setting, assumptions, algorithm, and/or model. [Yes]
\item An analysis of the properties and complexity (time, space, sample size) of any algorithm. [Yes]
\item (Optional) Anonymized source code, with specification of all dependencies, including external libraries. [Yes]
\end{enumerate}

\item For any theoretical claim, check if you include:
\begin{enumerate}
\item Statements of the full set of assumptions of all theoretical results. [Yes]
\item Complete proofs of all theoretical results. [Yes]
\item Clear explanations of any assumptions. [Yes]     
\end{enumerate}

\item For all figures and tables that present empirical results, check if you include:
\begin{enumerate}
\item The code, data, and instructions needed to reproduce the main experimental results (either in the supplemental material or as a URL). [Yes]
\item All the training details (e.g., data splits, hyperparameters, how they were chosen). [Yes]
\item A clear definition of the specific measure or statistics and error bars (e.g., with respect to the random seed after running experiments multiple times). [Yes]
\item A description of the computing infrastructure used. (e.g., type of GPUs, internal cluster, or cloud provider). [Yes]
\end{enumerate}

\item If you are using existing assets (e.g., code, data, models) or curating/releasing new assets, check if you include:
\begin{enumerate}
\item Citations of the creator If your work uses existing assets. [Yes]
\item The license information of the assets, if applicable. [Not Applicable]
\item New assets either in the supplemental material or as a URL, if applicable. [Yes]
\item Information about consent from data providers/curators. [Not Applicable]
\item Discussion of sensible content if applicable, e.g., personally identifiable information or offensive content. [Not Applicable]
\end{enumerate}

\item If you used crowdsourcing or conducted research with human subjects, check if you include:
\begin{enumerate}
\item The full text of instructions given to participants and screenshots. [Not Applicable]
\item Descriptions of potential participant risks, with links to Institutional Review Board (IRB) approvals if applicable. [Not Applicable]
\item The estimated hourly wage paid to participants and the total amount spent on participant compensation. [Not Applicable]
\end{enumerate}
\end{enumerate}

\onecolumn
\appendix
\thispagestyle{empty}
\aistatstitle{Supplementary Materials}

\addtocontents{toc}{\protect\setcounter{tocdepth}{2}}
\tableofcontents

\section{\textsf{iProxCMPG}: FULL STOCHASTIC ALGORITHM}
\label{sec:appendix-full-stochastic-alg}

In this section, for the convenience of the reader, we report the full pseudo-code of Algorithm~\ref{alg:iprox-cmpg} in the stochastic setting where exact gradients are not available. See Algorithm~\ref{alg:iprox-cmpg-sto}.

\begin{algorithm*}[h]
\caption{\textsf{iProxCMPG}: \textbf{i}ndependent \textbf{Prox}imal-policy algorithm for \textbf{CMPG}s}
\label{alg:iprox-cmpg-sto}
\begin{algorithmic}[1]
\State \textbf{initialization:} $\pi^{(0)} \in \Pi^{\xi}$ s.t.\ $V_c(\pi^{(0)})<\alpha$ and suitably chosen $\eta,\beta,\xi,T,K,\{(\nu_k,\delta_k,\rho_k)\}_{0 \leq k \leq K}$
\For{$t=0,\dots,T-1$}
\State $\pi^{(t,0)}_i=\pi^{(t)}_i$
\For{$k=0,\dots,K-1$ and $i \in \mathcal N$ simultaneously}
\State sample $B$ trajectories $\{\{( a^{(b)}_{i,j},s^{(b)}_j,r^{(b)}_{i,j},c^{(b)}_j)\}_{j=0}^{\hat T^{(b)}_e}\}_{b=1}^B$ by following $\pi_i^{(t,k)}$
\State set $\hat V_{r_i}(\pi^{(t,k)})=\frac{1}{B}\sum_{b=1}^B\sum_{j=0}^{\hat T^{(b)}_e}r^{(b)}_{i,j}$ and $\hat V_c(\pi^{(t,k)})=\frac{1}{B}\sum_{b=1}^B\sum_{j=0}^{\hat T_e^{(b)}}c^{(b)}_j$
\State $\hat\nabla V^{r_i}_{\pi_i}(\pi^{(t,k)})=\hat V_{r_i}(\pi^{(t,k)}) \cdot \frac{1}{B}\sum_{b=1}^B\sum_{j=1}^{\hat T_e^{(b)}}\nabla \log \pi_i(a^{(b)}_{i,j} \mid s_j^{(b)})$
\State $\hat\nabla V^c_{\pi_i}(\pi^{(t,k)})=\hat V_c(\pi^{(t,k)}) \cdot \frac{1}{B}\sum_{b=1}^B\sum_{j=1}^{\hat T_e^{(b)}}\nabla \log \pi_i(a^{(b)}_{i,j} \mid s_j^{(b)})$
\State $\pi^{(t,k+1)}_i= \begin{cases}
\mathcal P_{\Pi^{i,\xi}} \left[ \pi^{(t,k)}_i - \nu_k \hat{\nabla}_{\pi_i}V_{r_i}(\pi^{(t,k)})-\frac{\nu_k}{\eta} (\pi^{(t,k)}_i-\pi_i^{(t)}) \right] &\text{if } \hat{V}_c(\pi^{(t,k)}) + \beta - \alpha \leq \delta_k \\[4pt]
\mathcal P_{\Pi^{i,\xi}} \left[ \pi^{(t,k)}_i-\nu_k \hat{\nabla}_{\pi_i}V_{c}(\pi^{(t,k)})-\frac{\nu_k}{\eta} (\pi^{(t,k)}_i-\pi_i^{(t)}) \right] &\text{otherwise}\end{cases}$
\EndFor
\State $\mathcal B^{(t)}=\{ \lfloor K/2 \rfloor \leq k \leq K \mid \hat V_c(\pi^{(t,k)}) \leq \delta_k \}$
\State $\pi_i^{(t+1)}=\pi_i^{(t,\hat k)}$ where $\hat k=1$ if $\mathcal B^{(t)}=\emptyset$ and else sampled s.t.\ for $k \in \mathcal B^{(t)}$, $\mathbb P(\hat k = k)=\left( \sum_{k \in \mathcal B^{(t)}} \rho_k \right)^{-1} \rho_k$
\EndFor
\State \textbf{output:} $\pi_i^{(T)}$ for $i\in \mathcal N$
\end{algorithmic}
\end{algorithm*}

\begin{remark}
\label{remark:index-sampling}
For our analysis, the index $\hat k$ sampled in line 11 of \cref{alg:iprox-cmpg-sto} is supposed to be picked the same by all the players. This sampling step does not require any state-action-reward samples, it is just an index sampling that is useful for our proofs (even in the exact gradients case), see the constraint satisfaction inequality in the proof of Theorem~\ref{thm:csa-result}, page 28. 
\end{remark}

\section{PROOFS FOR SECTION~\ref{sec:analysis-complexity}}

\paragraph{Notation} For any integer~$n \geq 1$, we use the notation~$[n] := \{1, \dots, n \}$ throughout the proofs.

In this section, we provide complete proofs of our main results. We begin with the exact gradients case before addressing the more involved finite sample case.

\subsection{Proof of Theorem~\ref{thm:main-exact-gradients} --- Exact Gradients Case}
\label{sec:appendix-proof-exact-gradients}

First, we restate \cref{thm:main-exact-gradients}. 
\begin{T1}
Let \cref{ass:init-feasible,ass:uniform-slater} hold and let the distribution mismatch coefficient $D := \max_{\pi \in \Pi} \left\| d_{\mu}^{\pi}/\mu \right\|_{\infty}$ be finite. For any $\epsilon>0$, after running \textsf{iProxCMPG}, \cref{alg:iprox-cmpg}, with $\xi=0$, suitably chosen $\eta,\beta,T,K$, and $\{(\nu_k,\delta_k,\rho_k)\}_{0 \leq k \leq K}$, there exists $t \in [T]$, such that $\pi^{(t)}$ is a constrained $\epsilon$-NE in expectation\footnote{Notice that here we take the expectation w.r.t.\ the randomness which is induced by the sampling of $\hat k$ in line 11 of \cref{alg:iprox-cmpg-sto}.}. The total iteration complexity is given by $\mathcal O \left( \epsilon^{-4} \right)$ where $\mathcal O(\cdot)$ hides polynomial dependencies in $m,S,A_{\max},D, 1-\gamma$, and~$\zeta$.
\end{T1}

Before analyzing the outer loop of \cref{alg:iprox-cmpg}, we begin by focusing on the proximal-point update step. We first introduce some useful notation. Then, we explain how we can use the switching gradient algorithm in \cref{sec:appendix-csa} for approximately solving the proximal-point update step independently. We proceed by establishing guarantees that will be important in the analysis of the outer loop of \cref{alg:iprox-cmpg}.

\paragraph{Notation} Recall that for any policies~$\pi, \pi' \in \Pi$ and any~$\eta > 0$,  $\Phi_{\eta,\pi'}(\pi)=\Phi(\pi) + \frac{1}{2\eta} \left\| \pi-\pi' \right\|^2$, $V_{\eta,\pi'}^c(\pi)=V_c(\pi)+\frac{1}{2\eta} \left\| \pi-\pi' \right\|^2$ and $\Pi^c_{\eta,\pi'}=\left\{ \pi \in \Pi \mid V^c_{\eta,\pi'}(\pi) \leq \alpha-\beta \right\}$. Moreover, recall the following constrained optimization problem: 
\begin{align}
\label{eqn:prox-problem}
\min_{\pi \in \Pi^c_{\eta,\pi'}} \Phi_{\eta,\pi'}(\pi)\,.
\tag{ProxPb($\eta,\pi'$)}
\end{align}
In the following ``$\lesssim$'' denotes inequality up to numerical constants. Moreover, let $L_{\Phi}$ be the smoothness constant of the functions~$\Phi$ and~$V_c$ (see \cref{lem:basics}) and let $\Phi_{\max}$ be an upper bound\footnote{Such a bound is always trivially available.} on $\Phi$. Recall that under \cref{ass:init-feasible}, the initial policy~$\pi^{(0)}$ is strictly feasible. We denote the respective slack by $\bar\zeta_0>0$, i.e., $\bar\zeta_0 := \alpha - V_c(\pi^{(0)})$.

Next, we state and prove the guarantees provided by our proximal-point update subroutine.
\begin{lemma}
\label{lem:subproblem-guarantee}
Let \cref{ass:init-feasible} hold and let $0 < \bar\epsilon \leq \bar\zeta_0$. Set~$\beta=\bar\epsilon$, $\eta=\frac{1}{2L_{\Phi}}$, and $\xi=0$. Denote by~$\tilde\pi^{(t+1)}$ the unique optimal solution to (\hyperref[eqn:prox-problem]{ProxPb}($\eta,\pi^{(t)}$)). 
There exist $K = \mathcal O\left( \bar\epsilon^{\,-2} \right)$
and suitable choices of $\{(\nu_k,\delta_k,\rho_k)\}_{0 \leq k \leq K}$, such that lines 4-6 of \cref{alg:iprox-cmpg} guarantee that for any $t \in [T-1]$,
\begin{equation}
\label{eqn:subproblem-guarantee}
\begin{aligned}
\E \left[ \Phi_{\eta,\pi^{(t)}}(\pi^{(t+1)}) - \Phi_{\eta,\pi^{(t)}}(\tilde\pi^{(t+1)}) \right] &\leq \bar\epsilon^2, \\
\E \left[ V_c(\pi^{(t+1)}) \right] &\leq \alpha\,, 
\end{aligned}
\end{equation}
where the expectation is with respect to the randomness induced by the sampling of $\hat k$ in line 11 of \cref{alg:iprox-cmpg-sto}.
\end{lemma}
\begin{proof}
We divide the proof into two steps.
\begin{itemize}
\item\textbf{Step 1: Equivalent centralized update rule for our algorithm.} 
First, we argue that independently running the subroutine given by the inner loop of \cref{alg:iprox-cmpg}, i.e., lines 4-6, is equivalent to a centralized execution of the stochastic switching subgradient algorithm (see \cref{alg:csa}) applied to our proximal-point update problem. Crucially, as observed by \cite{leonardos_global_2021}, Proposition B.1, for any $i \in \mathcal N$ and $\pi \in \Pi$, it holds that $\nabla_{\pi_i} \Phi(\pi) = \nabla_{\pi_i} V_{r_i}(\pi)$.
We can extend this observation to our regularized potential and value functions, namely for any $\pi' \in \Pi$,
\begin{align*}
\nabla_{\pi_i} \Phi_{\eta,\pi'}(\pi) &= \nabla_{\pi_i} \Phi(\pi) + \frac{1}{\eta}\left( \pi_i-\pi_i' \right) \\
&= \nabla_{\pi_i} V_{r_i}(\pi) + \frac{1}{\eta}\left( \pi_i-\pi_i' \right),
\end{align*}
which is an expression that can be evaluated independently by player~$i$, since access to the joint policy $\pi$ is not required. Together with separability of the projection operator $\mathcal P_{\Pi^{\xi}}$, see e.g. \cite{leonardos_global_2021}, Lemma~D.1, we have 
\begin{align*}
\left( \mathcal P_{\Pi^{i,\xi}}\left[ \pi_i^{(t,k)} - \nu_k\nabla_{\pi_i}V^{r_i}_{\eta,\pi^{(t,k)}}(\pi^{(t,k)}) \right] \right)_{i \in \mathcal N} = \mathcal P_{\Pi^{\xi}} \left[ \pi^{(t,k)}-\nu_k\nabla_{\pi}\Phi_{\eta,\pi^{(t,k)}}(\pi^{(t,k)}) \right],
\end{align*}
and similarly, for the constraint value function,
\begin{align*}
\left( \mathcal P_{\Pi^{i,\xi}}\left[ \pi_i^{(t,k)} - \nu_k\nabla_{\pi_i}V^c_{\eta,\pi^{(t,k)}}(\pi^{(t,k)}) \right] \right)_{i \in \mathcal N} = \mathcal P_{\Pi^{\xi}} \left[ \pi^{(t,k)}-\nu_k\nabla_{\pi}V^c_{\eta,\pi^{(t,k)}}(\pi^{(t,k)}) \right].
\end{align*}
Moreover, since $V_c(\pi^{(t,k)})$ can be estimated equally by each player due to the cooperative nature of our constraint, we can conclude that \cref{alg:iprox-cmpg} is equivalent to a centralized version where the independent, simultaneous update in line 5 is replaced by the following centralized version:
\begin{align*}
\pi^{(t,k+1)}=
\begin{cases}
\mathcal P_{\Pi^{\xi}} \left[ \pi^{(t,k)} - \nu_k \hat{\nabla}_{\pi}\Phi_{\eta,\pi^{(t,k)}}(\pi^{(t,k)}) \right] &\text{if } \hat{V}_c(\pi^{(t,k)}) + \beta - \alpha \leq \delta_k, \\[6pt]
\mathcal P_{\Pi^{\xi}} \left[ \pi^{(t,k)}-\nu_k \hat{\nabla}_{\pi}V^c_{\eta,\pi^{(t,k)}}(\pi^{(t,k)}) \right] &\text{otherwise.}\end{cases}
\end{align*}

\item\textbf{Step 2: Induction on~$t$.}
Next, to prove the claimed guarantee for all $t \in [T-1]$, we proceed by induction on $t$. We will invoke results on the stochastic switching gradient algorithm (see CSA, \cref{alg:csa}) that are separately presented in \cref{sec:appendix-csa} in the context of constrained optimization. By \cref{ass:init-feasible}, since~$\bar\epsilon \leq \bar\zeta_0$ and $\beta=\bar\epsilon$, we have $V_c(\pi^{(0)}) \leq \alpha-\beta$.
That is, for $t=0$, the initial feasibility condition of our CSA result, \cref{thm:csa-result} in \cref{sec:appendix-csa}, holds for~$\pi^{(t)}$. Note further that in our deterministic case, \cref{ass:csa-estimators} (which is required for \cref{thm:csa-result}) holds, since by \cref{lem:basics} we have a bound on objective and constraint gradient norms.

Hence, we can apply \cref{thm:csa-result} in the deterministic setting, i.e., with batch size $J=1$ and access to exact gradients and constraint function values,
to $\Phi_{\eta,\pi^{(t)}}$ and $V^c_{\eta,\pi^{(t)}}$ with $\mu=L_{\Phi}$ and $M^2 \lesssim \max \left\{ M^2_G+\mu_G^2\Delta^4,M^2_F+\mu_F^2\Delta^4 \right\} \lesssim M_c^2+L_{\Phi}^2 \text{diam}(\Pi)^4\,,$
in the notation of \cref{thm:csa-result}. 
After plugging in the bounds on $M_c,L_{\Phi}$, and $\text{diam}(\Pi)$ from \cref{lem:basics}, and choosing $K$ as in the statement of this lemma, \cref{thm:csa-result} implies the desired bounds on constraint violation and optimality gap w.r.t.\ $\tilde \pi^{(t+1)}$ in~\cref{eqn:subproblem-guarantee}. This concludes the base case of the induction.

As induction hypothesis, suppose now that~\cref{eqn:subproblem-guarantee} holds for some $t \in [T-1]$. Then, due to $\beta \geq \bar\epsilon$, $V_c(\pi^{(t+1)})+\beta \leq \alpha+\bar\epsilon$ implies that the initial feasibility condition of \cref{thm:csa-result} is satisfied and hence with the same argument as above regarding \cref{ass:csa-estimators}, we can apply \cref{thm:csa-result} to conclude that at the end of iteration $t+1$ of \cref{alg:iprox-cmpg}, the inner loop guarantees that
\begin{align*}
\E \left[ \Phi_{\eta,\pi^{(t+2)}}(\pi^{(t+2)}) - \Phi_{\eta,\pi^{(t+2)}}(\tilde\pi^{(t+2)}) \right] &\leq \bar\epsilon^2, \\
\E \left[ V_c(\pi^{(t+2)}) \right] &\leq \alpha,
\end{align*}
i.e., the inductive hypothesis also holds for $t+1$.
\end{itemize}
\end{proof}

We next determine the number of iterations of the outer loop of \cref{alg:iprox-cmpg} required for convergence in the following sense.
\begin{lemma}
\label{lem:exact-gradients-main} 
Let $\epsilon>0$ and set~$\eta=\frac{1}{2L_{\Phi}}$. Suppose $K$ is chosen such that the guarantee from \cref{lem:subproblem-guarantee} holds for~$\bar\epsilon^2=\frac{\epsilon^2}{4\eta}$. Then, after $T=\frac{4\eta\Phi_{\max}}{\epsilon^2}$ iterations of the outer loop of \cref{alg:iprox-cmpg} where~$\Phi_{\max}$ is an upper bound of the potential function (i.e., $\forall \pi \in \Pi, \Phi(\pi) \leq \Phi_{\max}$), there exists $0 \leq t \leq T-1$ such that $\E[\|\pi^{(t+1)}-\pi^{(t)}\|] \leq \epsilon$.
\end{lemma}

\begin{remark}
Notice that we need an expectation in Lemma 2 in the exact gradients case because of the random index sampling in Algorithm~\ref{alg:iprox-cmpg}. See also Remark~\ref{remark:index-sampling}. 
\end{remark}

\begin{proof}
Let $\mathcal F_t$ denote the $\sigma$-field generated by the random variables given by the iterates $\pi^{(t)}$ up to iteration $t$. Notice that this randomness is induced by the sampling of $\hat k$ in line 11 of \cref{alg:iprox-cmpg-sto}.
By \cref{lem:subproblem-guarantee}, the inner loop of \cref{alg:iprox-cmpg} guarantees that for any~$0 \leq t \leq T-1,$
\begin{align*}
\E \left[ \Phi(\pi^{(t+1)})+\frac{1}{2\eta}\|\pi^{(t+1)}-\pi^{(t)}\|^2 \mid \mathcal F_t \right]
&= \E \left[ \Phi_{\eta,\pi^{(t)}}(\pi^{(t+1)}) \mid \mathcal F_t \right] \\
&\leq \E \left[ \Phi_{\eta,\pi^{(t)}}(\tilde\pi^{(t+1)}) \mid \mathcal F_t \right] + \bar\epsilon^2 \\
&\leq \Phi_{\eta,\pi^{(t)}}(\pi^{(t)}) + \bar\epsilon^2 \\
&= \Phi(\pi^{(t)}) + \bar\epsilon^2
\end{align*}
where the second inequality is due to $\E \left[  V^c_{\eta,\pi^{(t)}}(\pi^{(t)}) \mid \mathcal F_t \right]=\E \left[  V_c(\pi^{(t)}) \right] \leq \alpha$.
Taking total expectation in the above inequality, we obtain
\begin{align*}
\E \left[ \|\pi^{(t+1)}-\pi^{(t)}\|^2 \right] \leq 2\eta \left( \E \left[ \Phi(\pi^{(t)}) \right]-\E \left[ \Phi(\pi^{(t+1)}) \right] \right).
\end{align*}

Summing the above inequality over $0 \leq t \leq T-1$, using the upper bound~$\Phi_{\max}$ on the potential function and plugging in our choices of $\eta$, $T$, and $\bar\epsilon$, we obtain
\begin{align*}
\frac{1}{T} \sum_{t=0}^{T-1} \E \left[ \|\pi^{(t+1)}-\pi^{(t)}\|^2 \right] &\leq 2\eta \left( \bar\epsilon^2 + \frac{1}{T} \sum_{t=0}^{T-1} \E \left[ \Phi(\pi^{(t)}) \right] - \E \left[ \Phi(\pi^{(t+1)}) \right] \right) \\
&\leq \frac{2\eta\Phi_{\max}}{T} + 2\eta\bar\epsilon^2 \\
&\leq \epsilon^2.
\end{align*}
Using Jensen's inequality, we conclude that there exists~$t \in [T-1]$ such that $\E \left[ \|\pi^{(t+1)}-\pi^{(t)}\| \right] \leq \epsilon$.
\end{proof}

Next, we aim to prove that the event $\|\pi^{(t+1)}-\pi^{(t)}\|\leq \epsilon$ implies $\text{Nash-gap}(\pi^{(t+1)}) = \mathcal O(\epsilon)$ where the constrained Nash-gap is defined as
\begin{align}
\label{def:nash-gap}
\text{Nash-gap}(\pi^{*}):=\max_{i \in \mathcal N}\max_{\pi_i' \in \Pi_c^i(\pi_{-i}^{*})} V_{r_i}(\pi^{*})-V_{r_i}(\pi'_i,\pi^{*}_{-i})\,.
\end{align}
As a result, we will be able to argue that $\E \left[  \|\pi^{(t+1)}-\pi^{(t)}\| \right] \leq \epsilon$ implies $\E \left[ \text{Nash-gap}(\pi^{(t+1)}) \right] = \mathcal O(\epsilon)$, i.e., that the policy~$\pi^{(t+1)}$ is a constrained $\mathcal O(\epsilon)$-NE in expectation.

Towards this goal, we first show that a policy~$\pi^{(t+1)}$ satisfying $\|\pi^{(t+1)}-\pi^{(t)}\| = \mathcal O(\epsilon)$ (as in the previous lemma) is a~$\mathcal O(\epsilon)$-$\widetilde{\text{KKT}}$ policy for our initial constrained minimization problem. The $\epsilon$-$\widetilde{\text{KKT}}$ conditions are a slight modification of the standard $\epsilon$-KKT conditions adapted to our specific requirements (see \cref{def:kkt} and \cref{def:tilde-kkt} in \cref{sec:appendix-constr-opt-concepts}). 
In the following lemma, we will be referring to (in-)exact solutions as well as KKT and $\widetilde{\text{KKT}}$ conditions for different problems. Therefore, we first introduce additional useful notation for clarity. 

\paragraph{Notation} We refer to the following constrained optimization problem as \cref{eqn:init-problem}:
\begin{align}
\label{eqn:init-problem}
\min_{\pi \in \Pi} \Phi(\pi)\,. \tag{InitPb}
\end{align}
For the previously introduced (\hyperref[eqn:prox-problem]{ProxPb}($\eta,\pi^{(t)}$)), we distinguish between the inexact solution resulting from the update which we denote by~$\pi^{(t+1)}$, and the exact solution which will be denoted by~$\tilde\pi^{(t+1)}$ in the proof below. Furthermore, we define the Lagrangians for the two problems as
\begin{align*}
\mathcal L(\pi,\lambda)&=\Phi(\pi)+\lambda \left( V_c(\pi)-\alpha \right) \tag{InitPb--$\mathcal L$}\,, \\
\mathcal L_{\eta,\pi'}(\pi,\lambda)&=\Phi_{\eta,\pi'}(\pi)+\lambda \left( V^c_{\eta,\pi'}(\pi)-\alpha+\beta \right)\,. \tag{ProxPb$(\eta,\pi')$--$\mathcal L$}
\end{align*}

Using Lemma~\ref{lem:subproblem-guarantee} and Lemma~\ref{lem:exact-gradients-main}, the following lemma shows that \cref{alg:iprox-cmpg} is guaranteed to generate an $\mathcal O(\epsilon)$-$\widetilde{\text{KKT}}$ policy. Parts of the proof have appeared in a similar form in the optimization literature (see Lemma~3.5 and Theorem~3.2 in~\citet{jia_first-order_2023}, and Theorem~5 in~\citet{boob-et-al23mathprog}). The lemma below differs from these results, since we are in a smooth setting and prove convergence w.r.t.\ our notion of $\widetilde{\text{KKT}}$ conditions rather than towards a point that is near an $\epsilon$-KKT point. Moreover, our guarantee for the proximal update subroutine is somewhat weaker due to the relaxed constraint satisfaction condition that we use to switch between update types in the inner loop, see \cref{lem:subproblem-guarantee}. Additionally, in order to achieve exact primal feasibility (instead of $\epsilon$-approximate), we employ a feasibility margin~$\beta$.
\begin{lemma}
\label{lem:eps-KKT}
Let~\cref{ass:init-feasible,ass:uniform-slater} hold. Let $\epsilon>0$ and choose $\bar\epsilon,K,\beta$ as in \cref{lem:subproblem-guarantee,lem:exact-gradients-main}. 
If~$\pi^{(t+1)}$ is a policy such that $\|\pi^{(t+1)}-\pi^{(t)}\| \leq \epsilon$ for some~$t \in [T-1]$, then $\pi^{(t+1)}$ is a $(C_{\text{KKT}}\,\epsilon)$-$\widetilde{\text{KKT}}$ policy of \cref{eqn:init-problem} where $C_{\text{KKT}}$ is a positive constant such that $C_{\text{KKT}} \lesssim \frac{m^{2.5}A_{\max}^{1.5}S}{(1-\gamma)^{4.5}\sqrt{\zeta}}$.
\end{lemma}
\begin{proof}
First, note that (\hyperref[eqn:prox-problem]{ProxPb}($\eta,\pi^{(t)}$)) is a strongly convex optimization problem with strongly convex constraints, which is sufficient for the existence of a unique optimum $\tilde\pi^{(t+1)}$. Since by \cref{ass:uniform-slater}, Slater's condition holds for \cref{eqn:prox-problem} for any $\pi' \in \Pi$, strong duality is given for (\hyperref[eqn:prox-problem]{ProxPb}($\eta,\pi^{(t)}$)) and hence there exists a finite dual variable $\tilde\lambda^{(t+1)} \geq 0$ forming a KKT pair with $\tilde\pi^{(t+1)}$. We first claim that $\|\tilde\pi^{(t+1)}-\pi^{(t+1)}\| \leq \epsilon$. This can be seen as follows:
By optimality of $(\tilde\pi^{(t+1)},\tilde\lambda^{(t+1)})$ for (\hyperref[eqn:prox-problem]{ProxPb}($\eta,\pi^{(t)}$)), we have
\begin{align}
\label{eqn:subproblem-optimality-app-1}
\tilde\lambda^{(t+1)} \left( V^c_{\eta,\pi^{(t)}}(\tilde\pi^{(t+1)})-\alpha+\beta \right)&=0\,,  \\
\label{eqn:subproblem-optimality-app-2}
\left\langle \nabla_{\pi}\mathcal L_{\eta,\pi^{(t)}}(\tilde\pi^{(t+1)},\tilde\lambda^{(t+1)}), \pi^{(t+1)}-\tilde\pi^{(t+1)} \right\rangle &\geq 0.
\end{align}
From \cref{lem:basics}, we know that $\mathcal L_{\eta,\pi^{(t)}}(\cdot,\tilde\lambda^{(t+1)})$ is $L_{\Phi}(1+\tilde\lambda^{(t)})$-strongly convex. Therefore, after rearranging the standard strong convexity lower bound, we get
\begin{align*}
\frac{L_{\Phi}}{2}&\left( 1+\tilde\lambda^{(t+1)} \right) \| \tilde\pi^{(t+1)}-\pi^{(t+1)} \|^2 \\
&\leq \mathcal L_{\eta,\pi^{(t)}}(\pi^{(t+1)}, \tilde\lambda^{(t+1)})-\mathcal L_{\eta,\pi^{(t)}}(\tilde\pi^{(t+1)},\tilde\lambda^{(t+1)}) - \left\langle \nabla_{\pi} \mathcal L_{\eta,\pi^{(t)}}(\tilde\pi^{(t+1)},\tilde\lambda^{(t+1)}),\pi^{(t+1)}-\tilde\pi^{(t+1)} \right\rangle \\
&\overset{(a)}{\leq} \underbrace{\Phi_{\eta,\pi^{(t)}}(\pi^{(t+1)})-\Phi_{\eta,\pi^{(t)}}(\tilde\pi^{(t+1)})}_{\leq \bar\epsilon^2}+\tilde\lambda^{(t+1)}\underbrace{\left( V_c(\pi^{(t+1)})-\alpha-\beta \right)}_{\leq 0} \\
&\overset{(b)}{\leq} \bar\epsilon^2\,,
\end{align*}
where step (a) follows by applying \cref{eqn:subproblem-optimality-app-1,eqn:subproblem-optimality-app-2}, and step (b) by \cref{lem:subproblem-guarantee}, i.e.\ the guarantee for $\pi^{(t+1)}$ provided by the algorithm's inner loop.
Then, it follows from the previous inequality that 
\begin{align}
\label{eqn:eps-KKT-lemma-aux-bound}
\| \tilde\pi^{(t+1)}-\pi^{(t+1)} \| \leq \sqrt{\frac{2\bar\epsilon^2}{L_{\Phi}}} \leq \frac{\epsilon}{\sqrt{2}} \leq \epsilon.
\end{align}

Using the fact that $(\tilde\pi^{(t+1)},\tilde\lambda^{(t+1)})$ is a KKT pair for (\hyperref[eqn:prox-problem]{ProxPb}($\eta,\pi^{(t)}$)), we now argue that $(\pi^{(t+1)},\tilde\lambda^{(t+1)})$ is a $(C_{
\text{KKT}}\,\epsilon)$-$\widetilde{\text{KKT}}$ pair for \cref{eqn:init-problem} (see Definition~\ref{def:tilde-kkt} in Appendix~\ref{sec:appendix-constr-opt-concepts}). We check each one of the requirements of the definition in what follows.
\begin{itemize}
\item \textbf{Exact primal feasibility:} By \cref{lem:subproblem-guarantee}, we know that $V_c(\pi^{(t+1)}) \leq \alpha$ for any $0 \leq t \leq T-1$.
\item \textbf{Dual feasibility:} This immediately holds by dual feasibility of $(\tilde\pi^{(t+1)},\tilde\lambda^{(t+1)})$ for (\hyperref[eqn:prox-problem]{ProxPb}($\eta,\pi^{(t)}$)).
\item \textbf{Complementary slackness:} When $\tilde\lambda^{(t+1)}=0$, we clearly have $|\tilde\lambda^{(t+1)}\left( V_c(\pi^{(t+1)})-\alpha \right)|= 0 \leq \epsilon$. Otherwise, we have
\begin{equation}
\label{eqn:eps-KKT-lemm-aux-2}
\begin{aligned}
V_c(\pi^{(t+1)}) 
&\overset{(a)}{\geq} V_c(\tilde\pi^{(t+1)})-M_c\epsilon \\
&\overset{(b)}{=} \alpha-\beta-\frac{1}{2\eta}\| \tilde\pi^{(t+1)}-\pi^{(t)} \|^2 -M_c\epsilon \\
&\overset{(c)}{\geq} \alpha - \frac{\epsilon^2}{\eta} - \left( M_c+\frac{1}{2\sqrt{\eta}} \right) \epsilon\,,
\end{aligned}
\end{equation}
where (a) follows from $M_c$ Lipschitz continuity of~$V_c$ (see Lemma~\ref{lem:basics}-item (\ref{item:lipschitz})) and Eq.~(\ref{eqn:eps-KKT-lemma-aux-bound}), (b) stems from complementary slackness  of $(\tilde\pi^{(t+1)},\tilde\lambda^{(t+1)})$ for (\hyperref[eqn:prox-problem]{ProxPb}($\eta,\pi^{(t)}$)) which states that $V^c_{\eta,\pi^{(t)}}(\tilde\pi^{(t+1)})-\alpha+\beta=0$. To obtain inequality~(c), observe that using the bound from \cref{eqn:eps-KKT-lemma-aux-bound}, our assumption on $\|\pi^{(t+1)}-\pi^{(t)}\|$, and the triangle inequality, we have $\|\tilde\pi^{(t+1)}-\pi^{(t)}\| \leq \| \tilde\pi^{(t+1)}-\pi^{(t+1)} \| + \|\pi^{(t+1)}-\pi^{(t)}\| \leq 2\epsilon.$

Combining~\cref{eqn:eps-KKT-lemm-aux-2} with the upper bound $V_c(\pi^{(t+1)}) \leq \alpha$ from primal feasibility, we get
\begin{align}
\label{eqn:slackness-bound}
\left| \tilde\lambda^{(t+1)} \left( V_c(\pi^{(t+1)})-\alpha \right) \right| \leq \tilde\lambda^{(t+1)} \left( \frac{\epsilon^2}{\eta}+M_c\epsilon +\frac{\epsilon}{2\sqrt{\eta}} \right).
\end{align}
We now show that the dual variable~$\tilde\lambda^{(t+1)}$ is bounded by a constant depending on $\zeta$ using \cref{ass:uniform-slater} and strong duality. Indeed, we have 
\begin{equation}
\label{eqn:lambda-bound}
\begin{aligned}
\tilde\lambda^{(t+1)} \leq \frac{\left\| \nabla_{\pi} \Phi(\tilde\pi^{(t+1)}) \right\| + \eta^{-1} \left\| \tilde\pi^{(t+1)}-\pi^{(t)} \right\|}{\sqrt{\zeta\eta^{-1}}} 
\leq \frac{M_c+4\epsilon\eta^{-1}}{\sqrt{\zeta\eta^{-1}}}\,, 
\end{aligned}
\end{equation}
where the first inequality follows from the proof of Lemma~1 in~\citet{ma_quadratically_2023}, whereas the second inequality uses Lipschitzness of the potential function (see \cref{lem:basics}) and the fact that $\left\| \tilde\pi^{(t+1)}-\pi^{(t)} \right\| \leq 2\epsilon$. Combining~\cref{eqn:slackness-bound} and~\cref{eqn:lambda-bound}, and using the bounds on $M_c$ and $L_{\Phi}$ from \cref{lem:basics}, we obtain the desired~$C_\text{KKT}\epsilon$-complementary slackness.

\item \textbf{Variational Lagrangian stationarity:} Suppose by contradiction that the Lagrangian stationarity condition that comes with the $\frac{2\epsilon(1+\tilde\lambda^{(t+1)})}{\eta}$-KKT conditions does not hold for $\tilde\pi^{(t+1)}$ and \cref{eqn:init-problem}. Then there exists $\nu \in N_{\Pi}(\tilde\pi^{(t+1)},\tilde\lambda^{(t+1)})$ (normal cone to the convex set of policies~$\Pi$) such that
\begin{align*}
\nabla_{\pi} \mathcal L_{\eta,\pi^{(t)}}(\tilde\pi^{(t+1)},\tilde\lambda^{(t+1)})+\nu =0 \quad\text{and} \quad
\left\| \nabla_{\pi}\mathcal L(\tilde\pi^{(t+1)},\tilde\lambda^{(t+1)})+\nu \right\|&>\frac{2\epsilon(1+\tilde\lambda^{(t+1)})}{\eta}
\end{align*}
where the equality is by Lagrangian stationarity of $\tilde\pi^{(t+1)}$ for (\hyperref[eqn:prox-problem]{ProxPb}($\eta,\pi^{(t)}$)) and the inequality is due to the above assumed lack of Lagrangian stationarity of $\tilde\pi^{(t+1)}$ for \cref{eqn:init-problem}. Plugging in the definition of $\mathcal L_{\eta,\pi^{(t)}}(\tilde\pi^{(t+1)},\tilde\lambda^{(t+1)})$ and combining the equality and inequality above, one can conclude that
\begin{align*}
\frac{2\epsilon(1+\tilde\lambda^{(t+1)})}{\eta} < \left\| \nabla_{\pi}\mathcal L(\tilde\pi^{(t+1)},\tilde\lambda^{(t+1)})+\nu \right\| = \frac{1+\tilde\lambda^{(t+1)}}{\eta}\left\| \tilde\pi^{(t+1)}-\pi^{(t)} \right\|\,, 
\end{align*}
which contradicts the inequality $\left\| \tilde\pi^{(t+1)}-\pi^{(t)} \right\| \leq 2\epsilon$. Hence using the bound on $\tilde\lambda^{(t+1)}$ from \cref{eqn:lambda-bound}, the policy~$\tilde\pi^{(t+1)}$ is a $\tilde C \epsilon$-KKT policy for \cref{eqn:init-problem} with~$\tilde C = \frac{2}{\eta}\left(1+\frac{M_c+2\epsilon\eta^{-1}}{\sqrt{\zeta\eta^{-1}}}\right).$

By \cref{lem:KKT-to-max-product}, this implies that
\begin{align}
\label{eq:var-grad-interm}
\max_{\pi' \in \Pi} \left\langle \tilde\pi^{(t+1)}-\pi', \nabla_{\pi}\mathcal L(\tilde\pi^{(t+1)},\tilde\lambda^{(t+1)}) \right\rangle \leq \text{diam}(\Pi)\tilde C \epsilon.
\end{align}
Then, in view of showing the variational Lagrangian stationarity for the pair~$(\pi^{(t+1)},\tilde\lambda^{(t+1)})$ for \cref{eqn:init-problem},  we can write 
\begin{align}
\label{eq:var-lagr-interm2}
\max_{\pi' \in \Pi} &\left\langle \pi^{(t+1)}-\pi', \nabla_{\pi}\mathcal L(\pi^{(t+1)},\tilde\lambda^{(t+1)}) \right\rangle \nonumber\\
&= \max_{\pi' \in \Pi} \left\langle \tilde\pi^{(t+1)}-\pi', \nabla_{\pi}\mathcal L(\pi^{(t+1)},\tilde\lambda^{(t+1)}) \right\rangle + \left\langle \pi^{(t+1)}-\tilde\pi^{(t+1)},\nabla_{\pi}\mathcal L(\pi^{(t+1)},\tilde\lambda^{(t+1)}) \right\rangle \nonumber\\
&\leq \max_{\pi' \in \Pi} \left\langle \tilde\pi^{(t+1)}-\pi', \nabla_{\pi}\mathcal L(\pi^{(t+1)},\tilde\lambda^{(t+1)}) \right\rangle + \left\| \pi^{(t+1)}-\tilde\pi^{(t+1)} \right\| \cdot \left\| \nabla_{\pi}\mathcal L(\pi^{(t+1)},\tilde\lambda^{(t+1)}) \right\| \nonumber\\
&\leq \max_{\pi' \in \Pi} \left\langle \tilde\pi^{(t+1)}-\pi', \nabla_{\pi}\mathcal L(\pi^{(t+1)},\tilde\lambda^{(t+1)}) \right\rangle + \epsilon (1+\tilde\lambda^{(t+1)})M_c\,. 
\end{align}

We now bound the first term in the above inequality by using $2L_{\Phi}(1+\tilde\lambda^{(t+1)})$-smoothness of $\nabla_{\pi}\mathcal L(\cdot,\tilde\lambda^{(t+1)})$ and~(\ref{eq:var-grad-interm}).
Using the fact that $\max_{\pi \in \Pi} \left( A(\pi)+B(\pi) \right) \leq \max_{\pi \in \Pi}A(\pi)+\max_{\pi \in \Pi}B(\pi)$ for any functions $A(\pi),B(\pi)$, we have
\begin{align*}
&\max_{\pi' \in \Pi} \left\langle \tilde\pi^{(t+1)}-\pi', \nabla_{\pi}\mathcal L(\pi^{(t+1)},\tilde\lambda^{(t+1)}) \right\rangle \\
&\leq \max_{\pi' \in \Pi} \left\langle \tilde\pi^{(t+1)}-\pi', \nabla_{\pi}\mathcal L(\pi^{(t+1)},\tilde\lambda^{(t+1)})-\nabla_{\pi}\mathcal L(\tilde\pi^{(t+1)},\tilde\lambda^{(t+1)}) \right\rangle 
+ \max_{\pi' \in \Pi} \left\langle \tilde\pi^{(t+1)}-\pi', \nabla_{\pi}\mathcal L(\tilde\pi^{(t+1)},\tilde\lambda^{(t+1)}) \right\rangle \\[4pt]
&\leq \max_{\pi' \in \Pi} \left\| \tilde\pi^{(t+1)}-\pi' \right\| \cdot 2L_{\Phi}(1+\tilde\lambda^{(t+1)}) \left\| \tilde\pi^{(t+1)}-\pi^{(t+1)} \right\| + \text{diam}(\Pi)\tilde C \epsilon \\
&\leq \left( 2 \text{diam}(\Pi) L_{\Phi}(1+\tilde\lambda^{(t+1)})+\text{diam}(\Pi) \tilde C \right) \epsilon\,.
\end{align*}
Combining the above inequality with~(\ref{eq:var-lagr-interm2}), we obtain 
\begin{align*}
\max_{\pi' \in \Pi} &\left\langle \pi^{(t+1)}-\pi', \nabla_{\pi}\mathcal L(\pi^{(t+1)},\tilde\lambda^{(t+1)}) \right\rangle 
\leq \left( (1+\tilde\lambda^{(t+1)})M_c + 2 \text{diam}(\Pi) L_{\Phi}(1+\tilde\lambda^{(t+1)}) + \text{diam}(\Pi) \tilde C \right) \epsilon.
\end{align*}
\end{itemize}
Finally, we use the bound on $\tilde\lambda^{(t+1)}$ from \cref{eqn:lambda-bound}, as well as bounds on $\text{diam}(\Pi)$, $L_{\Phi}$, and $M_c$ from \cref{lem:basics}, to conclude that~$\pi^{(t+1)}$ is a $C_{\text{KKT}}\epsilon$-$\widetilde{\text{KKT}}$ policy for \cref{eqn:init-problem}.
\end{proof}

To complete the analysis, it now remains to show that an $\mathcal O(\epsilon)$-$\widetilde{\text{KKT}}$ policy of \cref{eqn:init-problem} is a constrained $\mathcal O(\epsilon)$-NE\@. 
For this, we leverage the playerwise gradient domination property satisfied by the potential function and the constraint value function. We first introduce some notations. 

\paragraph{Notation} For each player $i \in \mathcal N$ and each policy $\pi_{-i} \in \Pi^{-i}$, consider the playerwise constrained optimization problem given by 
\begin{align*}
\label{eqn:problem-playerwise}
\min_{\pi_i \in \Pi^i_c(\pi_{-i})} V_{r_i} \left( \pi_i,\pi_{-i} \right)\,. 
\tag{PlayerPb($\pi_{-i}$)}
\end{align*}
The respective Lagrangian $\mathcal L_{\pi_{-i}}:\Pi^i \times \R_{\geq 0} \to \R$ is defined for every $\pi_i \in \Pi^i$ and every~$\lambda \geq 0$ by
\begin{align*}
\mathcal L_{\pi_{-i}} \left( \pi_i,\lambda \right)&=\Phi(\pi_i,\pi_{-i})+\lambda \left( V_c(\pi_i,\pi_{-i})-\alpha \right)\,.
\tag{PlayerPb($\pi_{-i}$)--$\mathcal L$}
\end{align*}

\begin{lemma}
\label{lem:KKT-cNE}
Let $\pi \in \Pi$ be an $\epsilon$-$\widetilde{\text{KKT}}$ policy of \cref{eqn:init-problem}. Then $\pi$ is a constrained $C_{\text{NE}}\,\epsilon$-NE where $C_{\text{NE}} \lesssim \frac{D}{1-\gamma}+1$.
\end{lemma}
\begin{proof}

The proof of the lemma proceeds in two steps: 
\begin{itemize}
\item \textbf{Step 1.} We show that if $\pi$ is an $\mathcal O(\epsilon)$-$\widetilde{\text{KKT}}$ policy of \cref{eqn:init-problem}, then for all $i \in \mathcal N$, $\pi_i$ is an $\mathcal O(\epsilon)$-$\widetilde{\text{KKT}}$ policy of \cref{eqn:problem-playerwise}.

\item \textbf{Step 2.} We conclude that each player cannot significantly improve its policy $\pi_i$ while staying within $\Pi_c^i(\pi_{-i})$ which means $\pi$ is a constrained $\mathcal O(\epsilon)$-NE.
\end{itemize}

We provide a proof of each one of the steps successively. 
\begin{itemize}
\item \textbf{Step 1:} Let $\lambda \geq 0$ be a dual variable such that $(\pi,\lambda)$ is an $\epsilon$-$\widetilde{\text{KKT}}$ pair of \cref{eqn:init-problem}, and let $i \in \mathcal N$ be arbitrary. We show that $(\pi_i,\lambda)$ is an $\epsilon$-$\widetilde{\text{KKT}}$ pair of \cref{eqn:problem-playerwise} by checking that the respective $\widetilde{\text{KKT}}$ conditions hold. For dual and exact primal feasibility, as well as complementary slackness, this is immediate since the conditions are equivalent for \cref{eqn:init-problem} and \cref{eqn:problem-playerwise}. For variational Lagrangian stationarity, observe that
\begin{align*}
\max_{\pi_i' \in \Pi^i} \left\langle \pi_i-\pi_i',\nabla_{\pi_i}\mathcal L_{\pi_{-i}}(\pi_i,\lambda) \right\rangle
&= \max_{\pi_i' \in \Pi^i} \left\langle \pi-(\pi_i',\pi_{-i}),\nabla_{\pi} \mathcal L(\pi,\lambda) \right\rangle \\
&\leq \max_{\pi' \in \Pi} \left\langle \pi-\pi',\nabla_{\pi} \mathcal L(\pi,\lambda) \right\rangle \\
&\leq \epsilon\,,
\end{align*}
where the first equality is due to $\nabla_{\pi_i} \mathcal L \left( \pi,\lambda \right) = \nabla_{\pi_i} \mathcal L_{\pi_{-i}} \left( \pi_i,\lambda \right)$ and the fact that all terms except for $\pi_i-\pi_i'$ vanish in the first argument of the scalar product. The second inequality is because $(\pi_i',\pi_{-i}) \in \Pi$, and the final step is by Lagrangian stationarity of $\pi$ for \cref{eqn:init-problem}.
\item \textbf{Step 2:} Let $i \in \mathcal N$ and consider the MDP $M_{\lambda}$, $\lambda \geq 0$, with state space $\mathcal S$, action space $\mathcal A_i$, probability transition kernel $P_{\lambda}$, reward $r_{\lambda}$, and initial distribution $\mu$ where
\begin{align*}
P_{\lambda}(s' \mid s, a_i) &:= \E_{a_{-i} \sim \pi_{-i}((a_i,\cdot) \mid s)}\left[ P(s' \mid s, (a_i,a_{-i})) \right] \\
r_{\lambda}(s,a_i) &:= \E_{a_{-i} \sim \pi_{-i}((a_i,\cdot) \mid s)} \left[ r_i(s,(a_i,a_{-i}))+ \lambda\, c(s,(a_i,a_{-i})) \right].
\end{align*}
Observe that $\mathcal L_{\pi_{-i}} \left( \pi_i,\lambda \right)$ is the value function associated to the policy~$\pi_i$ in the MDP $M_{\lambda}$, and hence gradient domination holds~\citep{agarwal-et-al21pg}, i.e., we have 
\begin{align*}
\mathcal L_{\pi_{-i}}(\pi_i,\lambda)-\min_{\tilde\pi_i \in \Pi^i} \mathcal L_{\pi_{-i}}(\tilde\pi_i,\lambda) 
\leq \frac{D}{1-\gamma} \max_{\pi_i' \in \Pi^i}\left\langle \pi_i-\pi_i', \nabla_{\pi_i}\mathcal L_{\pi_{-i}}(\pi_i,\lambda) \right\rangle\,,
\end{align*}
where $D$ is the distribution mismatch coefficient, supposed to be finite.
Using~\cref{prop:KKT-gr-dom} in Appendix~\ref{sec:appendix-constr-opt-concepts} for $C_1=0$, and using the definition of the playerwise primal optimum, see~\cref{eqn:problem-playerwise}, we then get
\begin{align*}
V_{r_i}(\pi)-\min_{\pi_i^{*} \in \Pi_{c_i}(\pi_{-i})}V_i(\pi_i^{*},\pi_{-i}) \leq \left( \frac{D}{1-\gamma}+1 \right)\epsilon.
\end{align*}
Since additionally we have exact primal feasibility of $\pi$ for \cref{eqn:init-problem}, the result follows by definition of the constrained $\epsilon$-NE.
\end{itemize}
\end{proof}
Finally, we put together above lemmas to prove the main theorem.
\begin{proof}[Proof of \cref{thm:main-exact-gradients}]
Let $\bar\epsilon^2 = \frac{L_{\Phi}\epsilon^2}{C_{\text{KKT}}^2C_{\text{NE}}^2}$. Then with $K=\mathcal O \left( \epsilon^{-2} \right)$, and $T=\mathcal O \left( \epsilon^{-2} \right)$, \cref{lem:subproblem-guarantee} and \cref{lem:exact-gradients-main} imply that there exists $0 \leq t \leq T-1$ such that $\E \left[  \left\| \pi^{(t+1)}-\pi^{(t)} \right\| \right] \leq \frac{\epsilon}{C_{\text{KKT}}C_{\text{NE}}}$. We use \cref{lem:eps-KKT} to conclude that $\pi^{(t+1)}$ is a $\epsilon/C_{\text{NE}}$-$\widetilde{\text{KKT}}$ policy of \cref{eqn:init-problem}. Then, by \cref{lem:KKT-cNE}, $\pi^{(t+1)}$ is a constrained $\epsilon$-NE\@. The total iteration complexity is bounded by $T \cdot K = \mathcal O \left( \epsilon^{-4} \right)$.
\end{proof}

\subsection{Proof of Theorem~\ref{thm:main-finite-sample} --- Finite Sample Case}
\label{sec:appendix-proof-finite-sample}

Moving on to the stochastic setting, we first restate \cref{thm:main-finite-sample}.
\begin{T2}
Let \cref{ass:init-feasible,ass:uniform-slater} hold, and let $D$ (as in \cref{thm:main-exact-gradients}) be finite. Then, for any $\epsilon > 0$, after running \textsf{iProxCMPG} based on finite sample estimates (see \cref{alg:iprox-cmpg-sto}) for suitably chosen $\eta,\beta,\xi,T,K,B$, and $\{(\nu_k,\delta_k,\rho_k)\}_{0 \leq k \leq K}$, there exists $t \in [T]$, such that in expectation, $\pi^{(t)}$ is a constrained $\epsilon$-NE\@. The total sample complexity is given by $\tilde{\mathcal O} \left( \epsilon^{-7} \right)$
where $\tilde{\mathcal{O}}(\cdot)$ hides polynomial dependencies in $m,S,A_{\max},D,1-\gamma$, and~$\zeta$, as well as logarithmic dependencies in $1/\epsilon$.
\end{T2}

Similar to the deterministic case, we begin by proving the guarantees provided by the inner loop of \cref{alg:iprox-cmpg-sto}.

\begin{lemma}
\label{lem:subproblem-guarantee-sto}
Let \cref{ass:init-feasible} hold, let $\bar\epsilon>0$ and set~$\beta=\bar\epsilon, \eta=\frac{1}{2L_{\Phi}}, \xi=\bar\epsilon\sqrt{2\eta}$. Then, there exist $K=\tilde{\mathcal O}\left( \bar\epsilon^{\,-3} \right)$, $B = \tilde{\mathcal O}\left( \bar\epsilon^{\,-2} \right)$,
and suitable choices of $\{(\nu_k,\delta_k,\rho_k)\}_{0 \leq k \leq K}$, such that lines 4-11 of \cref{alg:iprox-cmpg-sto} guarantee that for any $t \in [T-1]$,
\begin{equation}
\label{eqn:subproblem-guarantee-sto}
\begin{aligned}
\E \left[ \Phi_{\eta,\pi^{(t)}}(\pi^{(t+1)}) - \Phi_{\eta,\pi^{(t)}}(\tilde\pi^{(t+1)}) \right]  &\leq \bar\epsilon^2, \\
\E \left[ V_c(\pi^{(t+1)}) \right] &\leq \alpha\,, 
\end{aligned}
\end{equation}
where $\tilde\pi^{(t+1)}$ denotes the unique optimal solution to (\hyperref[eqn:prox-problem]{ProxPb}($\eta,\pi^{(t)}$)).
\end{lemma}
\begin{proof}
The result follows similarly as for \cref{lem:subproblem-guarantee} in the deterministic case. We hence only point out differences. In order to ensure bounded norms of gradient estimates, we use $\xi$-greedy policies. Then, according to \cref{lem:variance-bound}, the second moment of value and constraint gradient estimates is bounded by $\mathcal O(1/\xi)$. The concentration result shown in \cref{lem:value-concentration} ensures that constraint value estimates follow a sub-exponential distribution. Therefore, \cref{ass:csa-estimators}, see \cref{sec:appendix-csa} on guarantees for our subroutine, is satisfied. We can thus apply the respective \cref{thm:csa-result}
for optimizing over $\Pi^{\xi}$, and with $\mu=L_{\Phi}$ and $M^2 \lesssim \max \left\{ M^2_G+\mu_G^2\Delta^4,M^2_F+\mu_F^2\Delta^4 \right\} \lesssim \frac{24A_{\max}^2}{\xi(1-\gamma)^4}+L_{\Phi}^2 \text{diam}(\Pi)^4.$
After plugging bounds on $L_{\Phi}$, and $\text{diam}(\Pi)$ from \cref{lem:basics}, and choosing $K$ and $B$ as stated, \cref{thm:csa-result} implies the desired bounds via the same arguments as in the proof of \cref{lem:subproblem-guarantee}.
\end{proof}

Next, we analyze the convergence of our main proximal-point method, \cref{alg:iprox-cmpg-sto}. More concretely, we bound the required sample complexity for ensuring that for some $\epsilon>0$, there exist iterates $\pi^{(t)},\pi^{(t+1)}$ such that~$\left\| \pi^{(t)}-\pi^{(t+1)} \right\| = \mathcal O(\epsilon)$. In the following, we will then prove that this implies convergence to a constrained~$\mathcal O(\epsilon)$-NE.

Similarly to the deterministic case, we next determine the number of updates needed until convergence in the following sense. The next lemma is analogous to its deterministic counterpart Lemma~\ref{lem:exact-gradients-main}. 
\begin{lemma}
\label{lem:finite-sample-main}
Let $\epsilon>0$ and set $\eta=\frac{1}{2L_{\Phi}}$. Suppose $K$ is chosen such that the guarantee from \cref{lem:subproblem-guarantee-sto} holds for~$\bar\epsilon^2=\frac{\epsilon^2}{4\eta}$. Then after $T=\frac{4\eta\Phi_{\max}}{\epsilon^2}$ iterations of the outer loop of \cref{alg:iprox-cmpg-sto}  where~$\Phi_{\max}$ is an upper bound of the potential function (i.e., $\forall \pi \in \Pi, \Phi(\pi) \leq \Phi_{\max}$), there exists $0 \leq t \leq T-1$ such that $\E \left[  \|\pi^{(t+1)}-\pi^{(t)}\| \right] \leq \epsilon$.
\end{lemma}
\begin{proof}
The proof follows the same lines as the proof of Lemma~\ref{lem:exact-gradients-main} upon replacing \cref{lem:subproblem-guarantee} by \cref{lem:subproblem-guarantee-sto}. We do not reproduce it here for conciseness.
\end{proof}

Next, as in the exact gradients case, we aim to prove that the event $\|\pi^{(t+1)}-\pi^{(t)}\|\leq \epsilon$ implies $\text{Nash-gap}(\pi^{(t+1)}) = \mathcal O(\epsilon)$, in order to argue that $\E \left[  \|\pi^{(t+1)}-\pi^{(t)}\| \right] \leq \epsilon$ implies $\E \left[ \text{Nash-gap}(\pi^{(t+1)}) \right] = \mathcal O(\epsilon)$.

Recall that in \cref{lem:eps-KKT} we have already shown $\|\pi^{(t+1)}-\pi^{(t)}\|\leq \epsilon$ to imply that $\pi^{(t+1)}$ is a $C_{\text{KKT}}\epsilon$-$\widetilde{\text{KKT}}$ policy of \cref{eqn:init-problem} which equivalently holds in the stochastic $\xi$-greedy setting.
Arguing that a $\epsilon$-$\widetilde{\text{KKT}}$ policy is a constrained $\mathcal O(\epsilon)$-NE however requires an adapted proof, since here in each iteration we solve the subproblem over $\Pi^\xi$ instead of $\Pi$, i.e., the $\widetilde{\text{KKT}}$ conditions hold w.r.t.\ $\Pi^\xi$. The following lemma is an adjustment of \cref{lem:KKT-cNE} for this fact.
\begin{lemma}
\label{lem:KKT-cNE-sto}
Let $\pi \in \Pi^{\xi}$ be an $\epsilon$-$\widetilde{\text{KKT}}$ policy of \cref{eqn:init-problem} (where $\widetilde{\text{KKT}}$ are w.r.t.\ $\Pi^\xi$) and $\xi=\epsilon$. Then $\pi$ is a constrained $\hat C_{\text{NE}}\,\epsilon$-NE where $\hat C_{\text{NE}} \lesssim \frac{D}{1-\gamma}+\frac{m \sqrt{S}A_{\max}D}{(1-\gamma)^{4.5}}+1$.
\end{lemma}
\begin{proof}
We divide the proof into two steps:
\begin{itemize}
\item \textbf{Step 1:} Analogously to step 1 of \cref{lem:KKT-cNE}, one can show that $(\pi_i,\lambda)$ is an $\epsilon$-$\widetilde{\text{KKT}}$ pair of \cref{eqn:problem-playerwise}.
\item \textbf{Step 2:} Let $i \in \mathcal N$ and consider the MDP $\tilde M_{\lambda}$ (for $\lambda \geq 0$) with state space $\mathcal S$, action space $\mathcal A_i$, transition probability kernel~$\tilde{P}$, reward $\tilde r_{\lambda}$, discount factor $\gamma$, and initial distribution $\mu$ where
\begin{align*}
\tilde{P}(s' \mid s, a_i) &:= \E_{a_{-i} \sim \pi_{-i}(\cdot \mid s)}\left[ P(s' \mid s, (a_i,a_{-i})) \right]\,, \\
\tilde r_{\lambda}(s,a_i) &:= \E_{a_{-i} \sim \pi_{-i}(\cdot \mid s)} \left[ r_i(s,(a_i,a_{-i}))+\lambda\, c(s,(a_i,a_{-i})) \right].
\end{align*}
Observe that $\mathcal L_{\pi_{-i}} \left( \pi_i,\lambda \right)$ is the value function associated to the policy~$\pi_i$ for $\tilde M_{\lambda}$, and hence gradient domination holds~\citep{agarwal-et-al21pg}. In our particular case of $\pi$ being a $\xi$-greedy policy, we can also show a similar inequality, even w.r.t.\ the non-$\xi$-greedy optimum. Let
\begin{align*}
\hat\pi_i &:= \argmax_{\pi_i' \in \Pi^i} \left\langle  \pi_i-\pi'_i, \nabla_{\pi_i}\mathcal L_{\pi_{-i}}(\pi_i,\lambda) \right\rangle, \\
\hat\pi^{\xi}_i &:= \argmax_{\pi'_i \in \Pi^{i,\xi}} \left\langle  \pi_i-\pi'_i, \nabla_{\pi_i}\mathcal L_{\pi_{-i}}(\pi_i,\lambda) \right\rangle.
\end{align*}
Then, we have
\begin{align*}
\mathcal L_{\pi_{-i}}(\pi_i,\lambda)&-\min_{\pi_i^{*} \in \Pi^i} \mathcal L_{\pi_{-i}}(\pi_i^{*},\lambda) \\
&\leq \frac{1}{1-\gamma} \left\| \frac{d_{\mu}^{\pi_i^{*},\pi_{-i}}}{\mu} \right\|_{\infty} \left\langle  \pi_i-\hat\pi_i \nabla_{\pi_i}\mathcal L_{\pi_{-i}}(\pi_i,\lambda) \right\rangle \\
&= \frac{1}{1-\gamma} \left\| \frac{d_{\mu}^{\pi_i^{*},\pi_{-i}}}{\mu} \right\|_{\infty} \left\langle \pi_i-\hat\pi^{\xi}_i, \nabla_{\pi_i}\mathcal L_{\pi_{-i}}(\pi_i,\lambda) \right\rangle  + \left\langle \hat\pi_i^{\xi}-\hat\pi_i, \nabla_{\pi_i}\mathcal L_{\pi_{-i}}(\pi_i,\lambda) \right\rangle.
\end{align*}
We further bound the last term above as follows 
\begin{align*}
\left\langle \hat\pi_i^{\xi}-\hat\pi_i, \nabla_{\pi_i}\mathcal L_{\pi_{-i}}(\pi_i,\lambda) \right\rangle &= \max_{\pi_i^{\xi} \in \Pi^{i,\xi}} \left\langle \pi_i^{\xi}, \nabla_{\pi_i}\mathcal L_{\pi_{-i}}(\pi_i,\lambda) \right\rangle - \max_{\pi_i \in \Pi^i} \left\langle \pi_i, \nabla_{\pi_i}\mathcal L_{\pi_{-i}}(\pi_i,\lambda) \right\rangle \\
&\overset{(a)}{\leq} \xi \sqrt{S} \left\| \nabla_{\pi_i}\mathcal L_{\pi_{-i}}(\pi_i,\lambda) \right\| \\
&\overset{(b)}{\leq} \xi \sqrt{S} (1+\lambda)M_c\,. 
\end{align*}
In the above inequalities, (a) follows from using \cref{lem:greedy-decomposition} to obtain 
\begin{align*}
\max_{\pi_i^{\xi} \in \Pi^{i,\xi}} \left\langle \pi_i^{\xi}, \nabla_{\pi_i}\mathcal L_{\pi_{-i}}(\pi_i,\lambda) \right\rangle \leq \max_{\pi_i \in \Pi^i} \left\langle \pi_i, \nabla_{\pi_i}\mathcal L_{\pi_{-i}}(\pi_i,\lambda) \right\rangle + \sum_{a_i,s} \frac{\xi}{A_i} [\nabla_{\pi_i}\mathcal L_{\pi_{-i}}(\pi_i,\lambda)](a_i \mid s)\,. 
\end{align*}
Using the fact that for any $x \in \R^d$, $\|x\|_1 \leq \sqrt{d} \|x\|_2$, we further get
\begin{align*}
\sum_{a_i,s} \frac{\xi}{A_i} [\nabla_{\pi_i}\mathcal L_{\pi_{-i}}(\pi_i,\lambda)](a_i \mid s) = \frac{\xi}{A_i} \| \nabla_{\pi_i}\mathcal L_{\pi_{-i}}(\pi_i,\lambda) \|_1 \leq \frac{\xi}{A_i} \sqrt{A_i S}\, \| \nabla_{\pi_i}\mathcal L_{\pi_{-i}}(\pi_i,\lambda) \| \leq \zeta \sqrt S\, \| \nabla_{\pi_i}\mathcal L_{\pi_{-i}}(\pi_i,\lambda) \|\,.
\end{align*}
The bound used in (b) follows from Lipschitz continuity, see \cref{lem:basics}. We conclude that
\begin{align}
\label{eqn:lagrange-gr-dom-result}
\mathcal L_{\pi_{-i}}(\pi_i,\lambda)-\min_{\pi_i^{*} \in \Pi^i} \mathcal L_{\pi_{-i}}(\pi_i^{*},\lambda) 
\leq \frac{1}{1-\gamma} \left\| \frac{d_{\mu}^{\pi_i^{*},\pi_{-i}}}{\mu} \right\|_{\infty} \left[ \left\langle \pi_i-\hat\pi_i \nabla_{\pi_i}\mathcal L_{\pi_{-i}}(\pi_i,\lambda) \right\rangle + \xi \sqrt{S} (1+\lambda)M_c \right].
\end{align}
Applying~\cref{prop:KKT-gr-dom}, see \cref{sec:appendix-KKT-gr-dom}, with $\xi=\epsilon$, bounding the distribution mismatch coefficient by $D$, and using the definition of the playerwise primal optimum, see~\cref{eqn:problem-playerwise}, we then get
\begin{align*}
V_i(\pi)-\min_{\pi_i^{*} \in \Pi_{c_i}(\pi_{-i})}V_i(\pi_i^{*},\pi_{-i}) &\leq \left( \frac{D}{1-\gamma}+\frac{(1+\lambda)\sqrt{S}M_c D}{1-\gamma}+1 \right)\epsilon \\
&\leq \left( \frac{D}{1-\gamma}+\frac{m \sqrt{S}A_{\max}D}{(1-\gamma)^{4.5}}+1 \right)\epsilon.
\end{align*}
where for the second inequality we use \cref{eqn:lambda-bound} and our bounds on $M_c$ and $L_{\Phi}$ from \cref{lem:basics}. Since additionally, we have exact primal feasibility of~$\pi$ for \cref{eqn:init-problem} by the $\widetilde{\text{KKT}}$ conditions, the result follows by definition of a constrained $\epsilon$-NE.
\end{itemize}
\end{proof}

Finally, we put together above lemmas to prove our main theorem in the stochastic setting.
\begin{proof}[Proof of \cref{thm:main-finite-sample}]
Let $\bar\epsilon^2 = \frac{L_{\Phi}\epsilon^2}{C_{\text{KKT}}^2\hat C_{\text{NE}}^2}$. Then with
$K=\tilde{\mathcal O} \left( \epsilon^{-3} \right)$,
$B=\tilde{\mathcal O} \left( \epsilon^{-2} \right)$,
and
$T=\mathcal O \left( \epsilon^{-2} \right)$,
\cref{lem:subproblem-guarantee} and \cref{lem:exact-gradients-main} imply that there exists $t \in [T-1]$ such that $\E \left[ \left\| \pi^{(t+1)}-\pi^{(t)} \right\| \right] \leq \frac{\epsilon}{C_{\text{KKT}}\hat C_{\text{NE}}}$. We use \cref{lem:eps-KKT} to conclude that $\pi^{(t+1)}$ is a $\epsilon/\hat C_{\text{NE}}$-$\widetilde{\text{KKT}}$ policy of \cref{eqn:init-problem} in expectation. Then, by \cref{lem:KKT-cNE}, $\pi^{(t+1)}$ is a constrained $\epsilon$-NE in expectation. The total sample complexity is bounded by $T \cdot K \cdot B = \tilde{\mathcal O} \left( \epsilon^{-7} \right)$.
\end{proof}

\subsection{Other Technical Lemmas}

The next lemma collects standard regularity properties of the value and potential functions.
\begin{lemma}
\label{lem:basics}
The following statements hold true.
\begin{enumerate}
\item\label{item:lipschitz} The functions~$\Phi$ and~$V_c$ are $M_c$-Lipschitz continuous over $\Pi$ with $M_c=\frac{\sqrt{mA_{\max}}}{(1-\gamma)^2}$. This immediately implies that $\left\| \nabla \Phi(\pi) \right\| \leq M_c$ and $\left\| \nabla V_c(\pi) \right\| \leq M_c$, for all $\pi \in \Pi$. 
\item\label{item:value-smoothness} For any~$i \in \mathcal N$ and any~$\pi_{-i} \in \Pi_{-i}$, the function $V_{r_i}\left( \cdot,\pi_{-i} \right)$ is $L_i$-smooth with~$L_i=\frac{2\gamma A_i}{(1-\gamma)^3}$ and hence~$L_i$-weakly convex.
\item\label{item:unregularized-smoothness} The functions~$\Phi$ and $V_c$ are $L_{\Phi}$-smooth with $L_{\Phi}=m \cdot \max_i L_i=\frac{2m\gamma A_{\max}}{(1-\gamma)^3}$ and hence~$L_{\Phi}$-weakly convex.
\item\label{item:regularized-smoothness} For $\eta=\frac{1}{2L_{\Phi}}$ and any $\pi' \in \Pi$, the regularized function~$\Phi_{\eta,\pi'}(\pi)=\Phi(\pi)+L_{\Phi}\left\| \pi-\pi' \right\|^2$ is $L_{\Phi}$-strongly convex and the functions~$\Phi_{\eta,\pi'}, V^c_{\eta,\pi'}$ are both $2L_{\Phi}$-smooth.
\item\label{item:lagrangian-smoothness} For any $\lambda \in \R,\pi' \in \Pi$ and $\eta=\frac{1}{2L_{\Phi}}$, $\mathcal L \left( \cdot,\lambda \right)=\Phi(\cdot)+\lambda V_c(\cdot)$ is $L_{\Phi}(1+\lambda)$-smooth, and $\mathcal L_{\eta,\pi'} \left( \cdot,\lambda \right)=\Phi_{\eta,\pi'}(\cdot)+\lambda V^c_{\eta,\pi'}(\cdot)$ is $2L_{\Phi}(1+\lambda)$-smooth. Hence $\mathcal L_{\eta,\pi'} \left( \cdot,\lambda \right)$ is also $L_{\Phi}(1+\lambda)$-strongly convex.
\end{enumerate}
\end{lemma}
\begin{proof}
Item~\ref{item:value-smoothness} has been proved in~\citet{agarwal-et-al21pg}, Lemma~D.3. Item~\ref{item:unregularized-smoothness} has been reported in~\citet{leonardos_global_2021}, Lemma~D.4. Item~\ref{item:regularized-smoothness} immediately follows from item~\ref{item:unregularized-smoothness}.
We now prove item~\ref{item:lagrangian-smoothness} for $\mathcal L$, the result for~$\mathcal L_{\eta,\pi'}$ follows similarly. Using the definition of the Lagrangian and the triangle inequality, for any $\lambda\in\R$ and $\pi,\pi'\in\Pi$,
\begin{align*}
\left\| \nabla_{\pi} \mathcal L \left( \pi,\lambda \right)-\nabla_{\pi} \mathcal L \left( \pi',\lambda \right) \right\|
&\leq \left\| \nabla \Phi(\pi)-\nabla \Phi(\pi') \right\| + \lambda \left\| \nabla V_c(\pi)-\nabla V_c(\pi') \right\| \\
&\leq L_{\Phi} \left\| \pi-\pi' \right\| + \lambda L_{\Phi} \left\| \pi-\pi' \right\| \\
&\leq L_{\Phi}(1+\lambda) \left\| \pi-\pi' \right\|.
\end{align*}
To show item~\ref{item:lipschitz}, Lipschitz continuity of $\Phi$ and $V_c$, observe that for any $i \in \mathcal N$, $\pi \in \Pi$ and $\pi'_i \in \Pi^i$, by using Lemma~32 of~\citet{zhang-et-al22softmax} in the second step, we have
\begin{align*}
\left| \Phi(\pi_i,\pi_{-i})-\Phi(\pi'_i,\pi_{-i}) \right|
&= \left| V_{r_i}(\pi_i,\pi_{-i})-V_{r_i}(\pi_i',\pi_{-i}) \right| \\
&\leq \frac{1}{(1-\gamma)^2} \max_{s \in \mathcal S} \left\| \pi_i(\cdot\mid s)-\pi'_i(\cdot\mid s) \right\|_1 \\
&\leq \frac{\sqrt{A_i}}{(1-\gamma)^2} \max_{s \in \mathcal S} \left\| \pi_i(\cdot\mid s)-\pi'_i(\cdot\mid s) \right\|_2 \\
&\leq \frac{\sqrt{A_i}}{(1-\gamma)^2} \left\| \pi_i-\pi'_i \right\|_2
\end{align*}
where in the third step we use the fact that for any $x \in \R^d$, $\left\| x \right\|_1 \leq \sqrt{d} \left\| x \right\|_2$. Then, the following decomposition yields the result: For any $\pi,\pi' \in \Pi$,
\begin{align*}
\left| \Phi(\pi)-\Phi(\pi') \right| &= \left| \sum_{i \in \mathcal N} \Phi(\pi_1',\dots,\pi_{i-1}',\pi_i,\pi_{i+1},\dots,\pi_m)-\Phi(\pi_1',\dots,\pi_{i-1}',\pi_i',\pi_{i+1},\dots,\pi_m) \right| \\
&\leq \sum_{i \in \mathcal N} \left| \Phi(\pi_1',\dots,\pi_{i-1}',\pi_i,\pi_{i+1},\dots,\pi_m)-\Phi(\pi_1',\dots,\pi_{i-1}',\pi_i',\pi_{i+1},\dots,\pi_m) \right| \\
&\leq \frac{1}{(1-\gamma)^2} \sum_{i \in \mathcal N} \sqrt{A_i} \left\| \pi_i-\pi'_i \right\| \\
&\leq \frac{\sqrt{mA_{\max}}}{(1-\gamma)^2} \left\| \pi-\pi' \right\|\,, 
\end{align*}
where in the second inequality we apply the above result for playerwise deviations, and the last step is again due to the fact that for any $x \in \R^d$, $\left\| x \right\|_1 \leq \sqrt{d} \left\| x \right\|_2$. The result follows similarly for~$V_c$.
\end{proof}

The next lemma is an immediate result showing that any $\xi$-greedy playerwise policy $\pi_i \in \Pi^{i,\xi}$ (see definition in the main part p.~7 which defines this set as a set of lower bounded policies away from zero) can be represented as a convex combination of a uniform distribution over the action space~$\mathcal{A}_i$ and a policy $\pi_i \in \Pi^i$.
\begin{lemma}
\label{lem:greedy-decomposition}
For any $\xi>0$, $i \in \mathcal N$,
\begin{align*}
\Pi^{i,\xi} \subseteq \left\{ \pi_i \in \Pi^i \mid \exists \pi_i' \in \Pi^i, \forall a_i \in \mathcal A_i, \forall s \in \mathcal S: \; \pi_i(a_i\mid s)=\xi/A_i+(1-\xi)\pi_i'(a_i\mid s) \right\}.
\end{align*}
\end{lemma}
\begin{proof}
Let $\xi>0$, $i \in \mathcal N$ and let $\pi_i^{\xi} \in \Pi^{i,\xi}$. Then for all $a_i \in \mathcal A_i$ and $s \in \mathcal S$, set
\begin{align*}
\pi_i(a_i\mid s) := \frac{\pi_i^{\xi}(a_i\mid s)-\xi/A_i}{1-\xi}.
\end{align*}
Indeed $\pi_i \in \Pi^i$, since for all $a_i \in \mathcal A_i$ and $s \in \mathcal S$,
we have $\pi_i(a_i\mid s) \geq 0$ due to $\pi_i^{\xi}(a_i\mid s) \geq \xi$ and
\begin{align*}
\sum_{a_i \in \mathcal A_i} \pi_i(a_i \mid s) = \frac{1}{1-\xi} \Bigg( \underbrace{\sum_{a_i \in \mathcal A_i} \pi_i^{\xi}(a_i\mid s)}_{=1} - \underbrace{\sum_{a_i \in \mathcal A_i} \xi/A_i}_{=\xi} \Bigg) = 1.
\end{align*}
\end{proof}

The following lemma shows that the estimators we use for the playerwise policy gradients are unbiased and enjoy a bounded variance. 
\begin{lemma}[\citet{daskalakis-foster-golowich20indep,leonardos_global_2021}]
\label{lem:variance-bound}
For any $\xi>0$ and $\pi \in \Pi^{\xi}$, we have 
\begin{align*}
\label{eqn:constraint-estimate-variance}
&\E_{\pi} \left[ \hat\nabla V^{r_i}_{\pi_i}(\pi) \right] = \nabla_{\pi_i}V_{r_i}(\pi) = \nabla_{\pi_i}\Phi(\pi)\,,\\
&\E_{\pi} \left[ \left\| \hat\nabla V^{r_i}_{\pi_i}(\pi) \right\|^2 \right] \leq \frac{24A_{\max}^2}{\xi(1-\gamma)^4}\,,\\
&\E_{\pi} \left[ \hat V_c(\pi)^2 \right] \leq \frac{1}{(1-\gamma)^2}\,.
\end{align*}
The same holds for $\hat\nabla V^c_{\pi_i}(\pi)$ w.r.t.\ $\nabla_{\pi_i}V_c(\pi)\,.$
\end{lemma}

Finally, the following lemma shows that our constraint function estimates concentrate around their mean.
\begin{lemma}
\label{lem:value-concentration}
For $\pi \in \Pi^{\xi}$, let $\hat V_c^{(1)},\dots,\hat V_c^{(B)}$ be independent copies of $\hat V_c(\pi)$, and let $\hat V_c:= \frac{1}{B}\sum_{i=1}^B\hat V_c^{(i)}$. Then there exists $C>0$ such that for any $\lambda \geq 0$,
\begin{align*}
\mathbb P \left( \left| \hat V_c-V_c(\pi) \right| > \frac{\lambda}{\sqrt{B}} \right) \leq 4\exp \left( -C(1-\gamma)\lambda \right) + 2\exp \left( -C^2(1-\gamma)^2\lambda^2 \right).
\end{align*}
\end{lemma}
\begin{proof}
For $i \in [B]$, we decompose $\hat V_c^{(i)}=\hat c_0^{(i)} + \hat V_{c,\geq 1}^{(i)}$ where $\hat c_0^{(i)}$ is the cost incurred at step 0, and let $\hat c_0:=\frac{1}{B}\sum_{i=1}^Bc_0^{(i)}$, $\hat V_{c,\geq 1}:=\frac{1}{B}\sum_{i=1}^B\hat V_{c,\geq 1}^{(i)}$. Since $0 \leq \hat c_0 \leq 1$, by Hoeffding's inequality, there exists $C_0>0$ such that for any $\lambda \geq 0$,
\begin{equation}
\mathbb P\left( \left| \hat c_0-\E[\hat c_0] \right| > \frac{\lambda}{2\sqrt{B}} \right) \leq 2\exp \left( -C_0\lambda^2 \right).
\label{eqn:hoeffding-result}
\end{equation}
Moreover, note that since for all $s \in \mathcal S,a \in \mathcal A$, $0 \leq c(s,a) \leq 1$, we have that for all $i \in [B]$, $V_{c,\geq 1}^{(i)} \leq T^{(i)}_e$ where $T^{(i)}_e$ is the stopping time of the respective episode and an independent copy of $T_e$. Assuming $\kappa_{s,a}=\min_{s \in \mathcal S,a \in \mathcal A}\kappa_{s,a}=1-\gamma$ for all $s \in \mathcal S,a \in \mathcal A$, $T_e$ follows a geometric distribution with parameter $1-\gamma$.
By definition of the geometric distribution and elementary computations, we get for any $\lambda \geq 0$, 
\begin{align*}
\mathbb P(T_e \geq \lambda) \leq \gamma^{\lceil \lambda \rceil} \leq \exp(\lceil \lambda \rceil \log \gamma) \leq \exp \left( -\lceil \lambda \rceil \frac{1-\gamma}{3} \right) \leq \exp \left( -(1-\gamma)\lambda/3 \right)
\end{align*}
which by a standard characterization of sub-exponential random variables, see Proposition~2.7.1 in~\citet{vershynin2018high}, implies that $T_e$, and therefore also $V_{c,\geq 1}^{(i)}$ for all $i \in [B]$ are sub-exponential. Moreover, by the so-called centering lemma for sub-exponential distributions, see section~2.7 in~\citet{vershynin2018high}, for any random variable~$X$ that is sub-exponential with parameter $\sigma$, there exists an absolute constant~$c$ such that~$X-\E[X]$ is sub-exponential with parameter~$c\sigma$. Thus for all $i \in [B]$, $V^{(i)}_{c,\geq 1}-\E[\hat V_{c,\geq 1}^{(i)}]$ is sub-exponential with parameter in~$\mathcal O(1/(1-\gamma))$. Then, we apply Bernstein's inequality, see Theorem~2.8.1 in~\citet{vershynin2018high}, to show that there exist $C_1,C_2>0$ such that for any $\lambda \geq 0$,
\begin{equation}
\label{eqn:bernstein-result}
\mathbb P \left( \left| \hat V_{c,\geq 1}-\E[\hat V_{c,\geq 1}] \right| > \frac{\lambda}{2\sqrt{B}} \right) \leq 2\exp \left( -C_1(1-\gamma)\lambda \right) + 2\exp \left( -C^2_2(1-\gamma)^2\lambda^2 \right).
\end{equation}
Finally, using a union bound we combine~\cref{eqn:hoeffding-result} and~\cref{eqn:bernstein-result} to get the desired bound.
\end{proof}

\section{STRONGLY CONVEX STOCHASTIC OPTIMIZATION WITH STRONGLY CONVEX EXPECTATION CONSTRAINT}\label{sec:appendix-csa}

In this section, we describe a stochastic gradient switching algorithm for stochastic constrained optimization under expectation constraints. Up to the modification of using a relaxed constraint (which is crucial for enabling its independent implementation), our algorithm and analysis follow the Cooperative Stochastic Approximation~(CSA) algorithm presented in~\citet{lan-zhou20csa} which is inspired by Polyak's subgradient method~\citep{polyak1967general}. \citet{lan-zhou20csa}~hint at the fact that a $1/K$ convergence rate of the CSA algorithm can be shown in the case of strongly convex objective and under expectation constraints. Here we explicitly carry out this analysis by deriving a result in expectation and under a somewhat weaker assumption on the distribution of the constraint function estimates. 

Let $X \subseteq \R^d$ be a convex and compact set with diameter $\Delta := \max_{x, x' \in X} \left\| x-x' \right\|$\,. 
Suppose $\theta$ are random vectors supported on $\Theta \subset \R^p$, and let $F:X \times \Theta \to \R, G:X \times \Theta \to \R$ be functions such that $F(\cdot,\theta)$ and $G(\cdot,\theta)$ are $\mu_F$ and $\mu_G$-weakly convex, respectively. For any $x' \in X$, we define $F_{\mu,x'}(x,\theta):=F(x,\theta)+\mu_F\|x-x'\|^2$ and $G_{\mu,x'}(x,\theta):=G(x,\theta)+\mu_G\|x-x'\|^2$. Let $f(x):=\E_{\theta}[F(x,\theta)],g(x):=\E_{\theta}[G(x,\theta)]$ (where expectations are supposed to be well-defined and finite) and $f_{\mu,x'}(x):=f(x)+\mu_F\|x-x'\|^2, g_{\mu,x'}(x):=g(x)+\mu_G\|x-x'\|^2$ for every~$x \in X\,.$ 

The problem we aim to solve\footnote{Note that the final guarantees we will obtain are actually in terms of a relaxed constraint satisfaction bound. This is due to our modification of the original CSA algorithm.} is given by
\begin{equation}
\label{eqn:csa-problem}
\begin{aligned}
\min_{x \in X} f_{\mu,x'}(x) &:= \E_{\theta}[F_{\mu,x'}(x,\theta)] \\
\text{s.t.\ } g_{\mu,x'}(x) &:= \E_{\theta}[G_{\mu,x'}(x,\theta)] \leq 0.
\end{aligned}
\end{equation}

Recall that such a problem needs to be solved at each time step in our \textsf{iProxCMPG} algorithm. The point~$x'$ is arbitrarily fixed throughout the rest of this section. 

Suppose we are only given access to first-order information of $f_{\mu,x'},g_{\mu,x'}$ and zeroth-order information of $g$ via a stochastic oracle that outputs unbiased and bounded-variance estimates. 
\begin{assumption}
\label{ass:csa-estimators} 
For every~$x \in X$, the estimators~$F_{\mu,x'}'(x,\theta), G'_{\mu,x'}(x,\theta)$ and~$G(x,\theta)$ are unbiased, i.e., $\E_{\theta} \left[ F_{\mu,x'}'(x,\theta) \right]=\nabla f_{\mu,x'}(x), \E_{\theta} \left[ G'_{\mu,x'}(x,\theta) \right]=\nabla g_{\mu,x'}(x)$ and $\E_{\theta} \left[ G(x,\theta) \right]=g(x)$\,. Moreover, there exist~$M_F,M_G>0$ such that
\begin{align*}
\E_{\theta} \left[ \left\| F'(x,\theta) \right\|^2 \right]\leq M_F^2\,;\quad
\E_{\theta} \left[ \left\| G'(x,\theta) \right\|^2 \right]\leq M_G^2\,.
\end{align*}
Furthermore, we suppose to have access to independent unbiased estimators $\hat G^{(1)},\dots,\hat G^{(J)}$ of $G(x,\cdot)$ for which there exists $\sigma>0$ such that for any $\lambda \geq 0$, it holds that
\begin{align}
\label{eqn:ass-csa-estimator}
\mathbb{P}_{\theta} \left(|\hat G-g(x)| > \lambda/\sqrt{J} \right) \leq 4\exp \left( -\lambda/\sigma \right) + 2\exp \left( -\lambda^2/\sigma^2 \right),
\end{align}
where $\hat G:=\frac{1}{J}\sum_{j=1}^J\hat G^{(j)}$.
\end{assumption}
It can be easily seen that \cref{ass:csa-estimators} also implies existence of $\tilde M_F,\tilde M_G$ such that
\begin{align*}
\E_{\theta} \left[ \left\| F'_{\mu,x'}(x,\theta) \right\|^2 \right]&\leq 2\E_{\theta} \left[ \left\| F'(x,\theta) \right\|^2 \right] + 2\mu^2_F \left\| x-x' \right\|^4 \leq 2M_F^2+2\mu_F^2\Delta^4 =: \tilde M_F^2,\\
\E_{\theta} \left[ \left\| G'_{\mu,x'}(x,\theta) \right\|^2 \right]&\leq 2\E_{\theta} \left[ \left\| G'(x,\theta) \right\|^2 \right] + 2\mu^2_G \left\| x-x' \right\|^4 \leq 2M_G^2+2\mu_G^2\Delta^4 =: \tilde M_G^2.
\end{align*}

\begin{remark}
Notice that the concentration requirement of~\cref{eqn:ass-csa-estimator} is relaxed compared to the sub-Gaussian assumption made in~\citet{lan-zhou20csa} which is too strong to hold in our case. We refer the reader to \cref{lem:value-concentration} where we prove that this weaker tail bound assumption holds for our constraint function estimates.
\end{remark}

\subsection{A Primal Gradient Switching Algorithm}
\cref{alg:csa} is designed as a primal algorithm that switches between taking a step along the objective or constraint gradient, depending on whether the constraint is currently (estimated to be) satisfied or not.

\begin{algorithm}[H]
\caption{CSA (adapted from \citet{lan-zhou20csa})}
\label{alg:csa}
\begin{algorithmic}[1]
\State \textbf{initialization:} $x_1 \in X$ s.t.\ $g(x_1) \leq \epsilon$ and $\left\{ \delta_k \right\}_{k \in [N]}, \left\{ \nu_k \right\}_{k \in [N]}, \left\{ \rho_k \right\}_{k \in [N]}, s \in [N]$
\For{$k=1,\dots,N-1$}
\State sample $\hat G_k^{(1)},\dots,\hat G_k^{(J)}$ from $G(x_k,\cdot)$ and set $\hat G_k=\frac{1}{J}\sum_{j=1}^J\hat G_k^{(j)}$
\State $x_{k+1} = \begin{cases} \mathcal P_X \left[ x_k-\nu_k F'_{\mu,x'}(x_k,\theta_k) \right] &\text{if } \hat G_k \leq \delta_k\\
\mathcal P_X \left[ x_k-\nu_k G'_{\mu,x'}(x_k,\theta_k) \right] &\text{else} \end{cases}$
\EndFor
\State let $\mathcal B_s:=\{ s \leq k \leq N \mid \hat G_k \leq \delta_k \}$
\State \textbf{output:} $x_{\hat k}$ where $\hat k=1$ if $\mathcal B_s=\emptyset$ and otherwise sampled s.t.\ for $k \in \mathcal B_s$, $\mathbb P(\hat k = k)=\left( \sum_{k \in \mathcal B_s} \rho_k \right)^{-1} \rho_k$
\end{algorithmic}
\end{algorithm}
In the analysis, we will denote $\mathcal M_s:=\{ s \leq k \leq N \mid k \not\in \mathcal B_s \}$ and $\mathcal B:=\mathcal B_1$, $\mathcal M:=\mathcal M_1$.

We point out the following differences between \cref{alg:csa} and the original CSA algorithm, see~\citet{lan-zhou20csa}, Algorithm 1.
\begin{enumerate}[label=(\alph*)]
\item We relax the switching condition in line 4 by using an estimate of $g(x_k)$ instead of~$g_{\mu,x'}(x_k)$ if we were to exactly use the algorithm proposed in~\citet{lan-zhou20csa}. This modification is crucial for our application as subroutine of an \emph{independent} learning algorithm, as described in the proof of \cref{lem:subproblem-guarantee}, see section~\ref{sec:appendix-proof-exact-gradients}. As a result, compared to~\citet{lan-zhou20csa}, we get a weaker guarantee in terms of constraint violation which however is still sufficient for our purposes.
\item Instead of constructing the output as a $\rho_k$-weighted average over iterates $x_k$, we sample an iterate from a $\rho_k$-weighted distribution, see line 6. This is because our relaxed constraint function $g$ is not necessarily convex (unlike $g_{\mu,x'}$) and hence we cannot easily bound the constraint value at an average over iterates.
\end{enumerate}

\subsection{Convergence and Sample Complexity Guarantee}
The following analysis uses the techniques presented in~\citet{lan-zhou20csa} applied to the strongly convex case with expectation constraint, under our modified \cref{ass:csa-estimators} and \cref{alg:csa}. The proofs follow along the same lines, we highlight differences when appropriate. 

First, we establish a basic recursion about CSA iterates that will be used repeatedly throughout the rest of the analysis.
\begin{proposition}
\label{prop:csa-1}
For any $s \in [N]$, $x \in X$, and $a_s$ as defined by~\cref{eqn:def-csa-weights}, it holds that
\begin{align*}
\sum_{k \in \mathcal M_s}&\rho_k \left( G_{\mu,x'}(x_k,\theta_k)-G_{\mu,x'}(x,\theta_k) \right)+\sum_{k \in \mathcal B_s}\rho_k \left( F_{\mu,x'}(x_k,\theta_k)-F_{\mu,x'}(x,\theta_k) \right) \\
&\leq (1-a_s)\Delta^2 + \frac{1}{2} \sum_{k \in \mathcal B_s}\rho_k\nu_k \left\| F'_{\mu,x'}(x_k,\theta_k) \right\|^2 + \frac{1}{2} \sum_{k \in \mathcal B_s}\rho_k\nu_k \left\| G'_{\mu,x'}(x_k,\theta_k) \right\|^2.
\end{align*}
\end{proposition}
\begin{proof}
Let $s \in [N]$ and $k \in \mathcal B_s$. Then, by non-expansiveness of the projection $\mathcal P_X$ and strong convexity,
\begin{align*}
\left\| x_{k+1}-x \right\|^2 &\leq \left\| x_k-x \right\|^2-\nu_k \left\langle F'_{\mu,x'}(x_k,\theta_k),x_k-x \right\rangle+\frac{1}{2}\nu_k^2 \left\| F'_{\mu,x'}(x_k,\theta_k) \right\|^2\\
&\leq \left\| x_k-x \right\|^2-\nu_k \left[ F_{\mu,x'}(x_k,\theta_k)-F_{\mu,x'}(x,\theta_k)+\frac{\mu_F}{2}\left\| x_k-x \right\|^2 \right] +\frac{1}{2}\nu_k^2 \left\| F'_{\mu,x'}(x_k,\theta_k) \right\|^2\\
&\leq \left( 1-\frac{\nu_k\mu_F}{2} \right)\left\| x_k-x \right\|^2-\nu_k \left[ F_{\mu,x'}(x_k,\theta_k)-F_{\mu,x'}(x,\theta_k) \right]+\frac{1}{2}\nu_k^2 \left\| F'_{\mu,x'}(x_k,\theta_k) \right\|^2.
\end{align*}
Similarly, if $k \in \mathcal M_s$,
\begin{align*}
\left\| x_{k+1}-x \right\|^2 \leq \left( 1-\frac{\nu_k\mu_G}{2} \right)\left\| x_k-x \right\|^2-\nu_k \left[ G_{\mu,x'}(x_k,\theta_k)-G_{\mu,x'}(x,\theta_k) \right]+\frac{1}{2}\nu_k^2 \left\| G'_{\mu,x'}(x_k,\theta_k) \right\|^2.
\end{align*}
After defining
\begin{equation}
\label{eqn:def-csa-weights}
\begin{aligned}
a_k&=\begin{cases}\mu_F\nu_k &\text{if } k \in \mathcal B\\\mu_G\nu_k &\text{if } k \in \mathcal M\end{cases}; \quad \quad
A_k=\begin{cases}1 &\text{if } k=1\\(1-a_k)A_{k-1} &\text{if } k \geq 2\end{cases}; \quad \quad
\rho_k=\frac{\nu_k}{A_k};
\end{aligned}
\end{equation}
the result follows by application of Lemma 21, \citet{lan-zhou20csa}.
\end{proof}

The next lemma provides a condition on $\{\nu_k,\delta_k,\rho_k\}_{s \leq k \leq N}$ that guarantees either low regret in terms of objective value or that a large number of iterates satisfy the constraint with high probability.
\begin{lemma}
\label{lem:csa-2}
Let $x^{*}$ be an optimal solution of \cref{eqn:csa-problem}. If for some $s \in [N]$ and $\lambda \geq 0$,
\begin{align}
\label{eqn:csa-condition}
\frac{N-s+1}{2}\min_{k \in \mathcal M_s}\rho_k\delta_k > (1-a_s)\Delta^2 + \frac{1}{2}\sum_{k \in \mathcal M_s}\rho_k\nu_k\tilde M_G^2+\frac{1}{2}\sum_{k \in \mathcal B_s}\rho_k\nu_k\tilde M_F^2+\frac{\lambda}{\sqrt J} \sum_{k \in \mathcal M_s}\rho_k,
\end{align}
then one of the following statements hold,
\begin{enumerate}
\item[(a)]\label{item:csa-case1} $\mathbb P_{\theta}(|\mathcal B_s| \geq (N-s+1)/2) \geq 1-|\mathcal M_s|\left(4\exp \left( -\lambda/\sigma \right) + 2\exp \left( -\lambda^2/\sigma^2 \right)\right)$, or,
\item[(b)]\label{item:csa-case2} $\sum_{k \in \mathcal B_s}\rho_k \left( f_{\mu,x'}(x_k)-f_{\mu,x'}(x^{*}) \right) \leq 0$.
\end{enumerate}
\end{lemma}
Note that unlike in~\citet{lan-zhou20csa}, due to our modified choice of \cref{alg:csa}'s output, well-definedness of $x_{\hat k}$ does not require $\mathcal B_s \not= \emptyset$.
\begin{proof}
In \cref{prop:csa-1}, set $x=x^{*}$, take expectation w.r.t.\ $\theta$ on both sides, and apply $\E_{\theta} \left\| F'_{\mu,x'}(x,\theta) \right\|^2 \leq \tilde M_F^2$, $\E_{\theta} \left\| G'_{\mu,x'}(x,\theta) \right\|^2 \leq \tilde M_G^2$. Then,
\begin{equation}
\label{eqn:csa-lemma-1}
\begin{aligned}
\sum_{k \in \mathcal M_s}\rho_k &\left( g_{\mu,x'}(x_k)-g_{\mu,x'}(x^{*}) \right)+\sum_{k \in \mathcal B_s}\rho_k \left( f_{\mu,x'}(x_k)-f_{\mu,x'}(x^{*}) \right) \\
&\leq (1-a_s)\Delta^2+\frac{1}{2}\sum_{k \in \mathcal M_s}\rho_k\nu_k\tilde M_G^2+\frac{1}{2}\sum_{k \in \mathcal B_s}\rho_k\nu_k\tilde M_F^2.
\end{aligned}
\end{equation}
If $\sum_{k \in \mathcal B_s}\rho_k \left( f_{\mu,x'}(x_k)-f_{\mu,x'}(x^{*}) \right) \leq 0$, then (b) holds. Otherwise, we make three observations. First, we have that $g_{\mu,x'}(x^{*}) \leq 0$. Second, it holds that $g(x_k) \leq g_{\mu,x'}(x_k)$. Third, for $k \in \mathcal M_s$, by \cref{ass:csa-estimators} and due to $\hat G_k>\delta_k$, we get
\begin{align}
\label{eqn:bernstein-app}
\mathbb{P}_{\theta}\left( g(x_k)<\delta_k-\frac{\lambda}{\sqrt{J}} \right)\leq 4\exp \left( -\lambda/\sigma \right) + 2\exp \left( -\lambda^2/\sigma^2 \right).
\end{align}
By a union bound this inequality holds for all $k \in \mathcal M_s$ with probability at most $|\mathcal M_s| \left( 4\exp \left( -\lambda/\sigma \right) + 2\exp \left( -\lambda^2/\sigma^2 \right) \right)$. Combining these three observations with \cref{eqn:csa-lemma-1} yields that with probability at least $1-|\mathcal M_s|\left(4\exp \left( -\lambda/\sigma \right) + 2\exp \left( -\lambda^2/\sigma^2 \right)\right)$, it holds that
\begin{align*}
\sum_{k \in \mathcal M_s}\rho_k \delta_k \leq (1-a_s)\Delta^2+\frac{1}{2}\sum_{k \in \mathcal M_s}\rho_k\nu_k\tilde M_G^2+\frac{1}{2}\sum_{k \in \mathcal B_s}\rho_k\nu_k\tilde M_F^2 + \frac{\lambda}{\sqrt{J}}.
\end{align*}
Above inequality then implies (a) because if $|\mathcal B_s|<(N-s+1)/2$, i.e., $|\mathcal M_s| \geq (N-s+1)/2$, then condition \cref{eqn:csa-condition} implies that
\begin{align*}
\sum_{k \in \mathcal M_s}\rho_k\delta_k \geq \frac{N-s+1}{2}\min_{k \in \mathcal M_s}\rho_k\delta_k > (1-a_s)\Delta^2 + \frac{1}{2}\sum_{k \in \mathcal M_s}\rho_k\nu_k\tilde M_G^2+\frac{1}{2}\sum_{k \in \mathcal B_s}\rho_k\nu_k\tilde M_F^2+\frac{\lambda}{\sqrt J} \sum_{k \in \mathcal M_s}\rho_k,
\end{align*}
which is a contradiction.
\end{proof}

Next, we state and prove the main guarantees provided by \cref{alg:csa}.
\begin{theorem}
\label{thm:csa-result}
Under \cref{ass:csa-estimators}, let $\epsilon>0$, suppose $x_1$ is such that $g(x_1) \leq \epsilon$, and let $f_{\max}>0$ such that for all $x \in X$, $0 \leq f(x) \leq f_{\max}$. Choose $s=N/2$,$\lambda=\sigma^2\log(N^2/(4f_{\max}))$, set $M=\max \{ \tilde M_G,\tilde M_F \}$, $\mu=\min \{ \mu_G,\mu_F \}$, and
\begin{align*}
\nu_k&=\begin{cases}\frac{2}{\mu_F(k+1)} &\text{if } k \in \mathcal B\\\frac{2}{\mu_G(k+1)} &\text{if } k \in \mathcal M\end{cases}; \quad \quad
\delta_k=\frac{\lambda}{\sqrt{J}}+\frac{1}{2k}\left( \frac{4\Delta^2}{k}+\frac{16M^2}{\mu^2} \right) \cdot \begin{cases} \mu_F &\text{if } k \in \mathcal B \\ \mu_G &\text{if } k \in \mathcal M \end{cases} ; \\
a_k&=\begin{cases}\mu_F\nu_k &\text{if } k \in \mathcal B\\\mu_G\nu_k &\text{if } k \in \mathcal M\end{cases}; \quad \quad
A_k=\begin{cases}1 &\text{if } k=1\\(1-a_k)A_{k-1} &\text{if } k \geq 2\end{cases}; \quad \quad
\rho_k=\frac{\nu_k}{A_k} \\
N &= \max \left\{ \frac{64\mu_FM^2}{\mu^2\epsilon^2},\frac{\sqrt{32\Delta^2\mu_F}}{\epsilon},\frac{32\sigma\mu_F}{\mu\epsilon^2} \right\}; \quad \quad J=\max \left\{ \frac{9\lambda^2}{\epsilon^2},\frac{32\sigma\mu_F}{\mu\epsilon^2} \right\}.
\end{align*}
Then \cref{alg:csa} guarantees that
\begin{align}
\label{eqn:csa-result-f}
\E \left[ f_{\mu,x'}(x_{\hat k})-f_{\mu,x'}(x^{*}) \right] &\leq \epsilon^2, \\
\label{eqn:csa-result-g}
\E \left[ g(x_{\hat k}) \right] &\leq \epsilon.
\end{align}
\end{theorem}
\begin{proof}
First, we observe that for any $k \in \mathcal M_s$,
\begin{equation}
\label{eqn:pr-to-e}
\begin{aligned}
\E \left[ \sqrt{J}\left( g(x_k)-\delta_k \right) \right]
&= \int_0^{\infty} \left(1-\mathbb{P} \left( \sqrt{J}\left( g(x_k)-\delta_k \right) \leq z \right) \right) dz \\
&\quad\quad-\int_{-\infty}^0\mathbb{P} \left( \sqrt{J}\left( g(x_k)-\delta_k \right) \leq z \right) dz \\
&\geq -\int_{-\infty}^0 4\exp \left( z/\sigma \right) + 2\exp \left( z^2/\sigma^2 \right) dz \\
&\geq -6\sigma
\end{aligned}
\end{equation}
where the first inequality is by \cref{eqn:bernstein-app}. Therefore, we have $\E[g(x_k)] \geq \delta_k-\frac{6\sigma}{\sqrt{J}}$. Moreover, by an argument similar to our derivation in \cref{lem:value-concentration} but with Bernstein's inequality applied to the sum $\left( \sum_{k \in \mathcal M_s}\rho_k \right)^{-1}\sum_{k \in \mathcal M_s}\rho_k\hat G_k$,
\begin{align*}
\mathbb P \left( \sum_{k \in \mathcal M_s}\rho_k \left( \hat G_k - g(x_k) \right) > \frac{\lambda}{\sqrt{J |\mathcal M_s|}}\sum_{k \in \mathcal M_s}\rho_k \right) \leq 4\exp \left( -\lambda/\sigma \right) + 2\exp \left( -\lambda^2/\sigma^2 \right).
\end{align*}
Therefore, following \cref{eqn:pr-to-e}, we get
\begin{align}
\label{eqn:jm-expectation-bound}
\E \left[ \sum_{k \in \mathcal M_s}\rho_kg(x_k) \right] \geq \sum_{k \in \mathcal M_s}\rho_k\delta_k - \frac{6\sigma}{\sqrt{J |\mathcal M_s|}}\sum_{k \in \mathcal M_s}\rho_k.
\end{align}

Next, we derive \cref{eqn:csa-result-f}. Note that \cref{eqn:csa-condition} holds for our choices of $s,\nu_k,\delta_k,\rho_k$. Then, if part (b) of \cref{lem:csa-2} holds, we have
\begin{align*}
\E \left[ f(x_{\hat k})-f(x^{*}) \right]
&= \E_{\hat k} \left[ \E \left[ f(x_k)-f(x^{*}) \mid \hat k = k \right] \right] \\
&\leq \left( \sum_{k \in \mathcal B_s}\rho_k \right)^{-1} \sum_{k \in \mathcal B_s} \rho_k \E \left[ f(x_k)-f(x^{*}) \right] \\
&\leq 0.
\end{align*}
Otherwise, if part (a) holds, then using above bound on $\E[g(x_k)]$ together with convexity of $f_{\mu,x'}$, \cref{eqn:csa-lemma-1} and \cref{eqn:jm-expectation-bound}, it follows that
\begin{align*}
\sum_{k \in \mathcal M_s} \rho_k\delta_k &- \frac{6\sigma}{\sqrt{J |\mathcal M_s|}}\sum_{k \in \mathcal M_s}\rho_k + \sum_{k \in \mathcal B_s}\rho_k \E \left[ f_{\mu,x'}(x_{\hat k})-f_{\mu,x'}(x^{*}) \right] \\
&\leq \sum_{k \in \mathcal M_s}\rho_k \E \left[ g(x_k) \right]+\sum_{k \in \mathcal B_s}\rho_k \E \left[ f_{\mu,x'}(x_{\hat k})-f_{\mu,x'}(x^{*}) \right] \\
&\leq \sum_{k \in \mathcal M_s}\rho_k \E \left[ g(x_k) \right]+\sum_{k \in \mathcal B_s} \rho_k \E \left[ f_{\mu,x'}(x_k)-f_{\mu,x'}(x^{*}) \right] \\
&\leq (1-a_s)\Delta^2+\frac{1}{2}\sum_{k \in \mathcal M_s}\rho_k\nu_k\tilde M_G^2+\frac{1}{2}\sum_{k \in \mathcal B_s}\rho_k\nu_k\tilde M_F^2.
\end{align*}
Denote by $E_{\mathcal B_s}$ the event that $|\mathcal B_s| \geq (N-s+1)/2$. Then, using the law of total expectation, our choice of $\lambda=\sigma^2\log(N^2/(4f_{\max}))$, $\rho_k\delta_k\geq 0$, and above inequality, we have
\begin{align*}
\E &\left[ f(x_{\hat k})-f(x^{*}) \right] \\[5pt]
&\leq \E \left[ f(x_{\hat k})-f(x^{*}) \mid E_{\mathcal B_s} \right] \cdot \underbrace{\mathbb P \left( E_{\mathcal B_s} \right)}_{\leq 1} \;+\; \E \left[ f(x_{\hat k})-f(x^{*}) \mid \overline E_{\mathcal B_s} \right] \cdot \underbrace{\mathbb P \left( \overline E_{\mathcal B_s} \right)}_{\leq |\mathcal M_s| \left( 4\exp \left( -\lambda/\sigma \right) + 2\exp \left( -\lambda^2/\sigma^2 \right) \right)} \\
&\leq \left( \sum_{k \in \mathcal B_s}\rho_k \right)^{-1} \left( (1-a_s)\Delta^2+\frac{1}{2}\sum_{k \in \mathcal M_s}\rho_k\nu_k\tilde M_G^2+\frac{1}{2}\sum_{k \in \mathcal B_s}\rho_k\nu_k\tilde M_F^2 + \frac{6\sigma}{\sqrt{J |\mathcal M_s|}}\sum_{k \in \mathcal M_s}\rho_k \right)+\frac{1}{N} \\
&\leq \left( \frac{N-s+1}{2} \min_{k \in \mathcal B_s} \rho_k \right)^{-1} \left( (1-a_s)\Delta^2+\frac{1}{2}\sum_{k \in \mathcal M_s}\rho_k\nu_k\tilde M_G^2+\frac{1}{2}\sum_{k \in \mathcal B_s}\rho_k\nu_k\tilde M_F^2 + \frac{6\sigma}{\sqrt{J |\mathcal M_s|}}\sum_{k \in \mathcal M_s}\rho_k \right)+\frac{1}{N}.
\end{align*}
In order to show the constraint violation bound, note that by a similar argument as \cref{eqn:pr-to-e}, for any $k \in \mathcal B_s$, $\E [g(x_k)] \leq \delta_k+\frac{6\sigma}{\sqrt{J}}$, and therefore
\begin{align*}
\E \left[ g(x_{\hat k}) \right]
= \E_{\hat k} \left[ \E \left[ g(x_k) \mid k=\hat k \right] \right]
\leq \E_{\hat k} \left[ \delta_{\hat k} \right]+\frac{6\sigma}{\sqrt{J}}
= \frac{\sum_{k \in \mathcal B_s}\rho_k\delta_k}{\sum_{k \in \mathcal B_s}\rho_k} + \frac{6\sigma}{\sqrt{J}}.
\end{align*}

In order to derive the guarantees \cref{eqn:csa-result-f} and \cref{eqn:csa-result-g}, we plug in the choices of $\nu_k,\delta_k,a_k,\rho_k,N$, and $J$ stated in \cref{thm:csa-result}. Observing that for any $s \leq k \leq N$, we have $A_k=\frac{2}{k(k+1)}$, that for any $k \in \mathcal B$, we have $\rho_k=\frac{2k}{\mu_F}$ as well as $\rho_k\nu_k=\frac{4}{\mu_F^2}$, and for any $k \in \mathcal M$, $\rho_k=\frac{2k}{\mu_G},\rho_k\nu_k=\frac{4}{\mu_G^2}$.
\begin{align*}
\E \left[ f(x_{\hat k})-f(x^{*}) \right]
&\leq \frac{\Delta^2+2N\mu_F^{-2}\tilde M_F^2+2N\mu^{-2}_G\tilde M_G^2+\frac{2\sigma}{\sqrt{J}}N^{3/2}\mu_G^{-1}}{N^2/4 \cdot \mu_F^{-1}} \\
&\leq \frac{\Delta^2+4N\mu^{-2}M^2+\frac{2\sigma}{\sqrt{J}}N^{3/2}\mu^{-1}}{N^2/4 \cdot \mu_F^{-1}}+\frac{1}{N} \\
&\leq \frac{4\mu_F\Delta^2}{N^2} + \frac{16\mu_F\mu^{-2}M^2}{N} + \frac{8\sigma\mu^{-1}\mu_F}{\sqrt{JN}}+\frac{1}{N} \\
&\leq \epsilon^2/4 + \epsilon^2/4 + \epsilon^2/4 + \epsilon^2/4.
\end{align*}
Moreover, for the constraint bound, it holds that
\begin{align*}
\E \left[ g(x_{\hat k}) \right]
&\leq \frac{\sum_{k \in \mathcal B_s}\left( 4\Delta^2/k+16M^2/\mu^2 \right)}{\sum_{k \in \mathcal B_s} 2k/\mu_F} + \frac{6\sigma}{\sqrt{J}} \\
&\leq \frac{8\Delta^2\mu_F}{N^2}+\frac{16M^2\mu_F}{\mu^2N}+\frac{6\sigma}{\sqrt{J}} \\
&\leq \epsilon.
\end{align*}
\end{proof}

\section{BACKGROUND IN CONSTRAINED OPTIMIZATION AND A NOVEL TECHNICAL LEMMA}
\label{sec:appendix-constr-opt-concepts}

\paragraph{Notation} For any non-empty subset~$Y  \subset \R^d$ and any vector~$x \in \R^d$, the distance from~$x$ to the set~$Y$ is defined as~$\text{dist}(x,Y):= \inf_{y \in Y} \left\| x-y \right\|$ where~$\|\cdot\|$ is the standard 2-norm of the Euclidean space~$\R^d$.

In this section, we recall some useful definitions for constrained optimization. In particular, we recall the definition of an approximate Karush-Kuhn-Tucker (KKT) point and a variation thereof. Then we prove a new technical result that will be useful in our analysis. 

\subsection{Approximate KKT Points in Constrained Optimization}

Let~$X \subset \R^d$ be a closed convex set. Consider the following constrained optimization problem: 
\begin{equation}
\label{eqn:general-const-opt}
\tag{ConstrOpt}
\begin{aligned}
P^{*}=\,&\min_{x \in X} f(x) \\
&\,\text{s.t. }f_c(x) \leq 0\,, 
\end{aligned}
\end{equation}
where $f,f_c:X \to \R$ are differentiable (possibly nonconvex) functions. 

The associated Lagrangian function $\mathcal L:X \times \R_{\geq 0} \to \R$ is defined for any~$x \in X, \lambda \geq 0$ by $\mathcal L \left( x,\lambda \right)=f(x)+\lambda f_c(x)$\,. The primal and dual problems can be respectively  written as
\begin{align*}
P^{*} &= \inf_{x \in X} \sup_{\lambda \geq 0} \mathcal L \left( x,\lambda \right)\,,\\
D^{*} &= \sup_{\lambda \geq 0} \underbrace{\inf_{x \in X} \mathcal L \left( x,\lambda \right)}_{=:d \left( \lambda \right)}\,.
\end{align*}
By weak duality, we know that $P^{*} \geq D^{*}$.

For any~$x \in \R^d,$ the normal cone to the set~$X$ at~$x$ is defined by: 
\begin{align*}
N_X(x):=\left\{ g \in \R^d \mid \forall y \in X, \left\langle g, y - x \right\rangle \leq 0 \right\}.
\end{align*}
\begin{definition}
\label{def:kkt}
Let~$\epsilon \geq 0$. A point~$x \in X$ is an $\epsilon$-KKT point of \cref{eqn:general-const-opt} if there exists a real~$\lambda$ such that the following conditions hold:
\begin{align*}
f_c \left( x \right) &\leq \epsilon\,, \tag{primal feasibility} \\
\lambda &\geq 0\,, \tag{dual feasibility} \\
\left| \lambda f_c \left( x \right) \right| &\leq \epsilon\,, \tag{complementary slackness} \\
\text{dist} \left( \nabla_x \mathcal L(x,\lambda), -N_X(x) \right) &\leq \epsilon\,. \tag{Lagrangian stationarity}
\end{align*}
We also call $(x,\lambda)$ an $\epsilon$-KKT pair. 
The point~$x$ is simply a KKT point of \cref{eqn:general-const-opt} if moreover~$\epsilon = 0\,.$
\end{definition}

We additionally define a slight modification of the above standard KKT conditions which turns out to be useful in our analysis. More precisely, the definition replaces approximate Lagrangian stationarity by a variational form thereof. Moreover, primal feasibility is now supposed to be exact. Other conditions remain unchanged. 
\begin{definition}
\label{def:tilde-kkt}
Let~$\epsilon \geq 0\,.$ A point~$\tilde x \in X$ is an $\epsilon$-$\widetilde{\text{KKT}}$ point of~\cref{eqn:general-const-opt} if there exists a real~$\tilde\lambda$ such that the following conditions hold:
\begin{align*}
f_c \left( \tilde{x} \right) &\leq 0\,, \tag{exact primal feasibility} \\
\tilde\lambda &\geq 0\,, \tag{dual feasibility} \\
\left| \tilde{\lambda} f_c \left( \tilde{x} \right) \right| &\leq \epsilon\,, \tag{complementary slackness} \\
\max_{x' \in X} \left\langle \tilde x-x', \nabla_x \mathcal L(\tilde x,\tilde\lambda) \right\rangle &\leq \epsilon\,. \tag{variational Lagrangian stationarity}
\end{align*}
In particular, the point~$\tilde{x}$ is said to be a $\widetilde{\text{KKT}}$ point of \cref{eqn:general-const-opt} when~$\epsilon = 0$.
\end{definition}

The next lemma connects the first stationarity condition with a variational form thereof. In particular, this result allows to connect the two definitions of approximate KKT points above. 
\begin{lemma}
\label{lem:KKT-to-max-product}
Let $X \subseteq \R^d$ be a convex and compact set.
Let~$\epsilon > 0$ and let~$x, g \in \R^d\,.$ 
If~$\text{dist}\left( g,-N_X(x) \right) \leq \epsilon$, then~$\max_{x' \in X} \left\langle x-x',g \right\rangle \leq \Delta\epsilon\,,$ where $\Delta :=\max_{x,x' \in X}\left\| x-x' \right\|$ is the diameter of the set~$X\,.$
\end{lemma}

\begin{proof}
Let~$y_0 \in - N_{X}(x)\,.$ For any~$x' \in X,$ we have 
\begin{align*}
\ps{x-x', g} &= \ps{x-x', g-y_0} + \ps{-y_0, x'-x}\,,\\
&\leq \ps{x'-x, g-y_0}\,,\\
&\leq \|x'-x\| \cdot \|g - y_0\|\,,\\
\end{align*}
where the first inequality follows from the fact that~$y_0 \in - N_X(x)$, the second inequality stems from the Cauchy-Schwarz inequality. Taking the infimum with respect to~$y_0$ in the last inequality gives the desired inequality since $\text{dist}\left( g,-N_X(x) \right)=\inf_{y \in -N_X(x)}\left\| g-y \right\| \leq \eps\,.$
\end{proof}

\subsection{A Novel Technical Lemma for Approximate Optimality under Gradient Dominance}
\label{sec:appendix-KKT-gr-dom}

We now state our technical lemma. This results shows that an approximate KKT point of~\cref{eqn:general-const-opt} at which a gradient domination inequality holds for the Lagrangian function is approximately optimal for the objective function to be minimized.
\begin{proposition}
\label{prop:KKT-gr-dom}
Let $\epsilon>0$ and let~$\tilde x \in X$ be an~$\epsilon$-$\widetilde{\text{KKT}}$ point of \cref{eqn:general-const-opt}. Suppose there exist constants~$C_0,C_1 \geq 0$ such that the Lagrangian function associated to \cref{eqn:general-const-opt} satisfies for all $\lambda \geq 0$ and for all~$x \in X$,
\begin{align}
\mathcal L \left( x,\lambda \right)-\mathcal L \left( x^{*}_{\lambda},\lambda \right) \leq C_0 \max_{x' \in X}\left\langle x-x', \nabla_x \mathcal L(\tilde x,\tilde\lambda) \right\rangle + C_1\epsilon\,,
\label{eqn:lagrangian-approx-gr-dom}
\end{align}
where~$x^{*}_{\lambda}$ is a minimizer of $\mathcal L \left( \cdot,\lambda \right)$\,.
Then, we have 
\begin{align*}
f(\tilde x)-P^{*} \leq (C_0+C_1+1)\epsilon.
\end{align*}
\end{proposition}
\begin{proof}
Let $(\tilde{x},\tilde{\lambda})$ be an $\epsilon$-$\widetilde{\text{KKT}}$ pair. Then, we have
\begin{align*}
D^{*} \overset{(a)}{=} \max_{\lambda \geq 0} d(\lambda) &\geq d(\tilde{\lambda}) \\
&\overset{(b)}{=} \min_{x \in X} \mathcal L(x,\tilde{\lambda}) \\
&\overset{(c)}{\geq} \mathcal L(\tilde{x},\tilde{\lambda})-(C_0+C_1)\epsilon \\
&= f(\tilde{x})+\tilde\lambda f_c(\tilde x)-(C_0+C_1)\epsilon \\
&\overset{(d)}{\geq} f(\tilde{x})-(C_0+C_1+1)\epsilon \\
\end{align*}
where (a) and (b) are by definition, and (d) is due to complementary slackness. To see (c), observe that by Lagrangian stationarity and \cref{eqn:lagrangian-approx-gr-dom},
\begin{align*}
\epsilon \geq \max_{x' \in X}\left\langle \tilde x-x', \nabla_x \mathcal L(\tilde x,\tilde\lambda) \right\rangle \geq \frac{1}{C_0} \left( \mathcal L(\tilde{x},\tilde{\lambda})-\mathcal L(x^{*}_{\tilde{\lambda}},\tilde{\lambda}) - C_1 \epsilon \right)\,, 
\end{align*}
which implies that\footnote{If~$C_0 = 0$, the same inequality immediately holds from~(\ref{eqn:lagrangian-approx-gr-dom}).}
\begin{align*}
\mathcal L(\tilde{x},\tilde{\lambda})-\mathcal L(x^{*}_{\tilde{\lambda}},\tilde{\lambda}) \leq (C_0+C_1)\epsilon.
\end{align*}
Finally, we use weak duality, i.e.\ $P^{*} \geq D^{*}$, to conclude that
\begin{align*}
f(\tilde x)-P^{*} = \underbrace{f(\tilde x)-D^{*}}_{\leq (C_0+C_1+1)\epsilon} + \underbrace{D^{*}-P^{*}}_{\leq 0} \leq (C_0+C_1+1)\epsilon.
\end{align*}
\end{proof}

\section{ADDITIONAL DETAILS ABOUT SIMULATIONS}
\label{sec:additional-experiments}

We provide additional details regarding the implementation of our \textsf{iProxCMPG} (\cref{alg:iprox-cmpg-sto}) in practice:
\begin{enumerate}[label=(\alph*)]
\item In our experiments, each episode terminates after a fixed number of steps~$T_e = 10$ corresponding to a discount factor~$\gamma = 0.9\,.$ 
\item In order to reduce the variance and enable the usage of larger step sizes, all constraint and value (gradient) estimates are obtained by sampling a batch of $B$~trajectories. 
\item For the subroutine, i.e.\ as solution to the proximal-point update, we do not consider a $\rho_k$-weighted average over iterates but simply use the last iterate $\pi^{(t,K)}$.
\item We choose $\delta_k=0$ for all $k \in \mathbb{N}$.
\end{enumerate}

\begin{table}[h]
\caption{Overview of hyperparameters used in our simulations.}
\begin{center}
\begin{tabular}{@{}cccc@{}}
\hline
Hyperparameters & Number of players $m$ & Pollution tax & Energy marketplace \\ \hline\hline
Step size $\eta$ (outer loop) & - & 0.1 & 0.1 \\ \cline{2-4}
\multirow{3}{*}{Step size $\nu_k$ (inner loop)} & 2 & 0.005 & 0.002 \\
     & 4 & 0.002 & 0.001 \\
     & 8 & 0.0007 & 0.0003 \\ \cline{2-4}
\multirow{3}{*}{Sample batch size $B$} & 2 & 1000 & 100 \\
     & 4 & 1000 & 150 \\
     & 8 & 2500 & 200 \\ \cline{2-4}
$K$ (\#iterations inner loop) & - & 20 & 20 \\
$T$ (\#iterations outer loop) & - & 20 & 60 \\
Discount factor $\gamma$ & - & 0.9 & 0.9 \\
Episode length $T_e$ & - & 10 & 10 \\
\end{tabular}
\end{center}
\label{tab:params}
\end{table}

\paragraph{Hyperparameters} We report hyperparameter choices for our simulations in \cref{tab:params}. Note that to ensure convergence, as indicated by our theoretical results, a larger number of players $m$ requires smaller step sizes and larger sample batches. Step sizes~$\eta$ and~$\nu_k$ were chosen by tuning over the range $[0,1]$.

\paragraph{Error Bars and Reproducibility} The plots in \cref{fig:pollution-sim-plot,fig:der-sim-plot} show the means of estimated potential values across 10 independent runs, and the corresponding shaded region displays the respective standard deviation. Obtaining results for all presented experiments thus requires simulating 60 runs in total. All experiments are fully reproducible using the provided code and specified seeds.

\paragraph{Computing Infrastructure} In order to reduce computation time by executing all runs in parallel, we conducted the simulations within less than 4 hours on a cluster of 15 4-core \emph{Intel(R) Xeon(R) CPU E3-1284L v4} clocked at 2.90GHz and equipped with 8Gbs of memory. 

\paragraph{Notation} As used in the main part, $\mathcal U(\{1, \cdots, W\})$ refers to the uniform distribution over the finite set~$\{1, \cdots, W\}$ where~$W \geq 2$ is an integer. 

\vfill

\end{document}